\documentclass{article}

\usepackage[numbers]{natbib}
\usepackage[final]{neurips_2025}

\usepackage{microtype}
\usepackage{graphicx}
\usepackage{booktabs} % for professional tables
\usepackage{definitions}
\usepackage{caption}
\usepackage{subcaption}
\usepackage{hyperref}
\usepackage{multirow}
\usepackage{amsmath}
\usepackage{amssymb}
\usepackage{mathtools}
\usepackage{amsthm}
\usepackage{amsfonts}       % blackboard math symbols

\usepackage[capitalize,noabbrev]{cleveref}

\theoremstyle{plain}
\newtheorem{theorem}{Theorem}[section]
\newtheorem{proposition}[theorem]{Proposition}
\newtheorem{lemma}[theorem]{Lemma}

\newtheorem{remark}[theorem]{Remark}
\theoremstyle{definition}
\newtheorem{definition}[theorem]{Definition}

\RequirePackage{fancyhdr}
\RequirePackage{xcolor} % changed from color to xcolor (2021/11/24)
\RequirePackage{algorithm}
\RequirePackage{algorithmic}
\RequirePackage{natbib}
\RequirePackage{eso-pic} % used by \AddToShipoutPicture
\RequirePackage{forloop}
\RequirePackage{url}

\title{Understanding Contrastive Learning via Gaussian Mixture Models}

\author{%
  Parikshit Bansal \\
  UT Austin \\
  \texttt{pbansal@utexas.edu}\\
  \And
  Ali Kavis \\
  UT Austin \\
  \texttt{kavis@austin.utexas.edu}\\
  \And
  Sujay Sanghavi \\
  UT Austin \\
  \texttt{sanghavi@mail.utexas.edu}\\
}

\begin{document}

\maketitle

\begin{abstract}
Contrastive learning involves learning representations via a loss function that encourages each (unlabeled) sample to be far from other samples, but close to its own \emph{augmentation}. 
In this paper, we aim to understand  why this simple idea performs remarkably well, by theoretically analyzing it for a simple, 
natural problem setting: dimensionality reduction in Gaussian Mixture Models (GMMs). 
Note that the standard GMM setup lacks the concept of augmentations. We study an intuitive extension: we define the pair of data sample and its augmentation as a coupled random draw from the GMM such that the marginal over the "noisy" augmentation is \emph{biased} towards the component of the data sample.
For this setup, we show that vanilla contrastive loss, e.g., InfoNCE, is able to find the \emph{optimal} lower-dimensional subspace even when the Gaussian components are non-isotropic.
In particular, we show that InfoNCE can match the performance of a fully supervised algorithm, e.g., LDA, (where each data point is labeled with the mixture component it comes from) even when the augmentations are "noisy".
We further extend our setup to the multi-modal case, and develop a GMM-like setting to study the contrastive CLIP loss. 
We corroborate our theory with experiments on CIFAR100; representations learned by InfoNCE loss match the performance of LDA on clustering metrics.

\end{abstract}

% \input{todo}

% \begin{itemize}
% \item Adapt paper for "factored gradient descscent"
% \item motivating figures
% \item Affine invariant. Can we construct interesting examples to demonstrate it's utility?
% \item \textbf{can we show examples for r<K when isotropic fails but our method works?}
% \end{itemize}

% Rewrite the paper with probably the main focus being "if we are given augmentations we can go meaningfully beyond SVD". 

\section{Introduction}
Contrastive learning (CL) is now a gold-standard paradigm for learning representations with popular examples such as vision models like SimCLR~\cite{pmlr-v119-chen20j} and MoCo~\cite{chen2020improved} , text models like Dense Passage Retrieval (DPR)~\cite{karpukhin2020dense}, and vision-language models like CLIP ~\cite{radford2021learning}. Once learned, these representations both demonstrate commendable zero-shot performance, and could easily be adapted to achieve SoTA performance on various tasks like classification~\cite{pmlr-v119-chen20j} and retrieval~\cite{izacard2021unsupervised}.

Contrastive learning is a self-supervised strategy that aims to learn coherent representations given an unlabeled dataset. 
The core idea is to relate \emph{similar} (positive) data samples to each other while contrasting \emph{unrelated} (negative) points. 
% to dissociate them. 
To formalize, consider a pair of points $(\vx, \hat{\vx})$, where $\vx$ is a sample from the dataset and $\hat{\vx}$ is a \enquote{similar}, positive example. Contrastive methods learn representations so that the embedding of the sample $\vx$ is \emph{close} to that of its partner $\hat{\vx}$, while being \emph{far away} from other points. Sometimes, these pairs are naturally present. For instance, in CLIP, each training data comes as a pair of an image and the corresponding text caption. In other scenarios, they can be manually obtained 
%and sometimes these pairs are made 
via data augmentation that preserves the underlying semantics, e.g., we can define an image and its resized/cropped/rotated version as a pair. 
% A fundamental question is to understand 

It is fundamentally important to understand {\bf why this simple idea of relating/contrasting \emph{pairs} works so well}. In this paper, we aim to answer this question with a theoretical study of contrastive learning by casting is as dimensionality reduction for Gaussian Mixture Models (GMMs). 
%\pari{should add not really just unlabeled, but unlabeled positives given. purely unlabeled only in vision}. The goal is to learn general purpose representations from massive amounts of \pari{semi-supervised} data, which can easily be adapted for downstream tasks. The core idea is to learn representations via an objective, which bring \textit{similar} points close together while pushing representations of random points far apart. This simple idea works remarkably well in various domains like vision~\cite{pmlr-v119-chen20j,chen2020improved,radford2021learning} and text~\cite{karpukhin2020dense}. 
We motivate our study through the lens of representation learning which 
% has inherent connection to our take on contrastive learning.
% Essentially, representation learning 
aims at finding a 
% (non-convex) 
mapping between the high-dimensional input space, e.g., images of 1 million pixels, and a lower dimensional output domain of \emph{representations}, e.g., 1024-dimensional feature vectors. Similarly, contrastive learning aims at discovering robust mappings between the data and the representations by exploiting additional information in the form of augmentations. In our framework of GMMs, we consider \emph{linear} mappings/projections (for analytical tractability) and characterize the \emph{optimality} of the mapping functions learned by contrastive methods by analyzing the associated projection subspaces. 

% An integral components of our work is a generalized notion of augmentations. In the context of GMMs, a natural choice of augmentation would be drawing a random sample from the \emph{same} component as the sample belongs to. We go beyond the idealized scenario; our augmentations are drawn from the GMM with a bias towards the component of the original sample and generated pairs might belong to different components with non-zero probability.
% outputs a lower dimensional vector (e.g. 1024 length vector) given a high-dimensional input (eg. a natural image of 1M pixels); we would like to understand why contrastive learning is able to find these functions well. 
% In this paper we consider functions that are linear projections; we show that contrastive learning is able to successfully learn \emph{optimal} -- more on this in \cref{sec:problem_statement}-- linear projections in GMMs where classical methods based on singular value decomposition (SVD) fail to do so. Key to our results and analysis is a novel way to formalize the notion of augmentations/pairs of points which contrastive learning leverages for learning favorable representations.

To highlight the success of contrastive methods in learning linear projectors, we conduct our study in comparison to fully supervised methods and standard spectral methods, i.e., based on singular value decomposition (SVD). 
% whereas CL methods find the optimal subspace. 
We provide an example where we provide a three-way comparison in finding anticipated projection subspaces (see \cref{subsec:parallel}). 
% At this point, we would like to emphasize the main difference between contrastive methods and SVD-based methods; 
% Note that contrastive learning methods leverage augmentations by design, but the SVD-based methods cannot. 
% Moreover, our endeavor is \emph{not} to determine the better of the methods. 
Note that SVD-based methods are not designed to leverage augmentations. 
Therefore, we use SVD-based methods merely as a reference point to argue that \emph{when (noisy) augmentations are present}, contrastive learning can \emph{provably} go beyond what would otherwise be possible. 
However, a more relevant and interesting comparison is with respect to supervised methods, i.e., LDA in this case. 
In that sense, we provide principled, theoretical insight into how contrastive learning works and how it performs with respect to a measure of optimality.
% ; we neither propose a new state-of-the-art method in dimensionality reduction for GMMs, nor identify weaknesses of spectral methods.

Within our GMM framework, a key proposal of ours is formalizing the \enquote{notion of augmentations}, i.e. point-pairs, going beyond the idealized definition considered in \citep{bizeul2024probabilistic}. In the context of GMMs, a natural choice of augmentation would be drawing a random sample \emph{from the same component} as the original point belongs to. Instead, we assume that for every point $\vx$, its augmentation $\hat{\vx}$ is a random draw \emph{from the GMM} following a distribution (see \cref{def:jointdistribution}) that is \emph{biased towards the component of the original point $\vx$}. 
% which belongs to a \enquote{tilted GMM} such that the component that $\vx$ comes from has a boosted probability. 
To keep the setting realistic, we are oblivious to which GMM component either of them comes from.

% Given pairs of points $(\vx,\hat{\vx})$ of the type described above, 
Given the model of noisy augmentations, we analyze contrastive learning objectives in two categories; single and multi-modal GMMs. In the single modal setting, the points in a pair are drawn from the same GMM distribution. We analyze InfoNCE and SimSiam losses, which are suited to leverage the additional information provided by the augmentations. 
In the {\bf multi-modal} context, the pair $(\vx,\hat{\vx})$ will be generated by \emph{separate} GMMs. For instance, in the CLIP~\citep{} images are associated with text descriptions such that we have pairs of the form $(\vx_V,\vx_T)$, where $\vx_V$ (images) and $\vx_T$ (texts) can have different dimensions. 
% ; for example, $\vx_V$ could be an image and $\vx_T$ a text string. 
The objective now is to learn two different projections, one for each modality to project them onto the \emph{same} space. 
% \ali{Parikshit, please take a look at the previous sentence} 
Contrastive methods aim to make the embedding of $\vx_V$ close to that of its pair $\vx_T$, and far from random other $\tilde{\vx}_T$. 
% \ali{Do we want to study multi-modal case with a particular focus on CLIP, or in a more general manner? We need to clarify.}
% in such a way that $\vx_V$ and $\vx_T$ come from
Since the samples come in natural pairs, we do not need augmentations. 
% \draft{Consider two {\em coupled} GMMs with $K$ components each in two different spaces, with a \emph{one-to-one} correspondence between their respective \emph{components}. Then, a joint sample $(\vx_V,\vx_T)$ from the coupled GMM is drawn 
% as follows. We randomly generate a $z \in [K]$ (with appropriate probabilities) and draw $\vx_V$ from the $z^{th}$ Gaussian of the \emph{first} GMM and $\vx_T$ from the $z^{th}$ Gaussian of the \emph{second} GMM.}
% We analyze the learning of linear projections using contrastive learning, as described in the above paragraph
We provide further details in \cref{sec:multimodal}.

{\bf Contributions:} Our objective is to theoretically understand the vital role of augmentations in self-supervised representation learning; we do so in the context of GMMs. To summarize,
\begin{enumerate}%[topsep=0.1cm]
    \item We formalize a generalized notion of \textbf{noisy augmentation/point pairs}, enabling us to concretely characterize contrastive learning methods in the context of GMMs. 
    % ; we informally described this above and formally do so below.
    \item We quantify a measure of optimality for the projections learned in the context of GMMs. We show that InfoNCE can find optimal projections for GMMs with \emph{general shared covariances}.
    % \draft{whereas spectral methods cannot do so beyond \emph{spherical} GMMs.}
    % For a GMM, the quality of a linear projection is measured by it's Fisher Discriminant. It is known that SVD-based methods can find the optimal linear projections for spherical GMMs, but cannot do so for more general GMMs. We show that linear projections learned using InfoNCE can find the optimal projections for general non-spherical GMMs with shared covariances. This result demonstrates (albeit in our simple setting) that when augmentations are present, contrastive learning can go beyond what is possible when they are absent. 
    % \ali{This is a long contribution. We should move the relevant background (first sentences) to somewhere before the contributions.}
    \item For the multi-modal setting, we show that the CLIP loss learns linear projections that are a subset of the (Fisher-)optimal subspaces for each modality, filtering out noise directions.
\end{enumerate}

\subsection{Related work} \label{sec:related_work}

\textbf{Contrastive learning and InfoNCE.} 
Contrastive learning is at the heart of state-of-the-art representation learning methods like SimCLR~\cite{pmlr-v119-chen20j}, MoCo~\cite{chen2020improved} etc.
Prior theoretical work on contrastive learning argue that the contrastive objective leads to learning of "general purpose" representations that facilitates better downstream performance and sample complexity. 
% Their analysis aim to show under some assumptions (on data distribution, augmentation generation and the function class of the learner) that the optimal solution of the contrastive objective is linearly adaptable for downstream tasks in a sample efficient manner. 
\cite{arora2019theoretical} gives the first theoretical analyses of the contrastive loss, showing provable guarantees on downstream task with reduced sample complexity. 
\cite{haochen2021provable} relaxes the augmentation assumptions made in \cite{arora2019theoretical} by showing that contrastive learning is equivalent to spectral clustering on some appropriately defined data graph. 
\cite{saunshi2022understanding,haochen2022theoretical} extend this line of work by arguing for the inclusion of the inductive biases of the neural network architecture into the theoretical analyses.  
\cite{johnson2022contrastive,zhai2023understanding} extend the setup of \cite{haochen2021provable} to show that the eigenfunctions of the positive pair graph (as defined in \cite{haochen2021provable}) are the basis of the set of view-invariant functions. They further show that these representations are equivalent to Kernel PCA with “positive-pair kernel”. \cite{parulekar2023infonce} show that under the condition that the augmentation sets within clusters are close and intertwined (relative to the complexity of the learning function class), the InfoNCE loss is able to learn consistent representations. 
% \cite{} also considers the Assumption 1.1 (“soft” view-invariance assumption) of Johnson et al. (2023) while making it homogeneous and their positive pair kernel.
% Based on some invariance/stability assumption on target functions (downstream losses), they argue about minimax optimality of SSL methods. However, we believe our work has important differences in the context of “view-invariance” idea (Assumption 1.1) presented in this work. Zhai et al., (2024) }

\textbf{Self-supervised learning.} 
BYOL~\cite{grill2020bootstrap}, SwAV~\cite{caron2020unsupervised}, SimSiam~\cite{chen2021exploring} are some of the most popular methods for self-supervised representation learning. They are also based on Siamese architecture used in contrastive methods, but notably, they do not make use of \emph{negatives} to prevent ``collapsing'' solutions. 
Contrastive methods like SimCLR~\cite{pmlr-v119-chen20j} and CLIP~\cite{radford2021learning} aim to learn representations so that embedding of a point is close to its partner while far away from that of any other point. On the contrary, \emph{non-contrastive} self-supervised learning methods only enforce the former condition and do not explicitly deal with the separation between the embeddings of samples. These methods are empirically motivated by the notion of breaking the symmetry between the encoders in Siamese networks~\cite{wang2022importance}. \cite{tian2021understanding,jing2021understanding,wen2022mechanism} propose various techniques for breaking this symmetry and study their training dynamics to understand their role in preventing collapse. Compared to the relevant work on self-supervised learning, we instead focus on the fixed point analysis of the loss functions.

\textbf{Dimensionality reduction for GMMs.} 
Spectral methods are popularly used for dimensionality reduction in GMMs by analyzing the principal directions of the data covariance. Spectral clustering algorithms~\cite{arora01,vempala2004spectral,achlioptas2005spectral,kannan2009spectral,brubaker2008isotropic} essentially combine spectral analysis with standard clustering algorithms (e.g. K-Means) in the low dimensional space.
Since their error bounds grow as square root of the ambient dimension, the two-step approach leads to lower clustering error than in the original space. 
\section{Problem setup}\label{sec:problem_statement}
% We setup the necessary notation and recap the problem of linear dimensionality reduction for gaussian mixture models. 
% We define dimensionality reduction as a linear projection onto a lower dimensional space. We formalize the quality of the lower dimensional space by fisher discriminant.  
% We also recall the standard technique for dimensionality reduction therein, which is based on SVD.=
% \paragraph{Input Data} 
We consider the task of learning a \enquote{good} representation function for points $\vx$ drawn from a data distribution. Informally, such a function should map the data onto a lower-dimensional subspace while preserving the class information. For tractability, we restrict $f$ to the class of linear mappings; 
$\vx \mapsto \vA^T\vx$,
where $\vA \in \mathbb{R}^{d\times r}$. In this setup, we will show that self-supervised methods can learn the optimal mapping for this task. 
% The standard approach for learning such a mapping is using spectral algorithms~\cite{kannan2009spectral} (See Sec~\ref{subsec:spectralmethods} for details).  
We compare and contrast our results for self-supervised methods to well-known results for supervised methods (e.g., LDA) and spectral methods. 
% Hence, spectral methods serve as a \emph{baseline} for interpreting our results.

As the basis of our study, we consider data distributions following a Gaussian mixture model (GMM). Intuitively, each mixture component in GMM represents a class in data and we assume that components have a shared covariance.
\begin{definition}%[SharedGMM]
\label{def:shared_gmm}
A \textit{Shared Covariance Gaussian Mixture Model} (SharedGMM) parameterized by $\{w_k,\vmu_k,\vSigma\}_{k\in[K]}$ is defined as the probability distribution 
    % \begin{align*}
    %     F = \sum_{k\in[K]} w_k \mathcal{N}(\vmu_k,\vSigma)
    % \end{align*}
$F = \sum_{k\in[K]} w_k \mathcal{N}(\vmu_k,\vSigma)$
where $\sum_k w_k = 1$, $\vmu_k$ are the means and $\vSigma \in \mathbb{S}_{++}^d$ is the covariance matrix shared by the components.
\end{definition}
To complete our data model, we formalize the definition of \emph{augmentation}. Given a point $\vx$, its augmentation $\hat \vx$ is an independent sample from the mixture distribution with a \textit{bias} towards the underlying component of $\vx$. The bias implies that if $\vx$ is sampled from a component $z$, then its augmentation, on average, is more likely to be sampled from the same component $z$.
% Recall the sampling process from a standard gaussian mixture model $F$, follows a two step process. We first sample a component $z \in [K]$, where component $k$ is selected with probability $w_k$. Then we sample a point from the gaussian corresponding to component $z$, i.e., $\vx|z \sim \mathcal{N}(\vmu_z,\vSigma_z)$. We want the augmentation for $\vx$ to be an independent sample from the mixture, with a boosted probability of belonging to component $z$. We formalize this notion below.
% $\aug{\vx}$ can intuitively be thought of as a sample from a conditional distribution (i.e. $\aug{\vx} \sim p(.|\vx)$).  
\begin{definition}%[\jointdistribution (AeD)]
\label{def:jointdistribution}
    For a SharedGMM $F$ parameterized by $\{w_k,\vmu_k,\vSigma\}_{k\in[K]}$, we define its \emph{\jointdistribution} (AeD) $\hat{F}$, parameterized by $\{w_k,\vmu_k,\vSigma, \delta\}_{k\in[K]}$ as 
    \begin{align*}
        \hat{F} = \delta \sum\nolimits_k w_k \big(\mathcal{N}(\vmu_k,\vSigma) \times \mathcal{N}(\vmu_k,\vSigma)\big) + (1-\delta) \big(\sum\nolimits_k w_k \mathcal{N}(\vmu_k,\vSigma) \times \sum\nolimits_{k'} w_{k'} \mathcal{N}(\vmu_{k'},\vSigma)\big)
    \end{align*}
\end{definition}
To elucidate, augmentation-enabled distribution (AeD) returns a pair of points $(\vx,\aug{\vx}) \sim \hat{F}$ in a two-step process. We first flip a coin with bias $\delta$. Based on the coin flip, we either sample twice from the same component or sample from possibly different components (chosen with probability $w_k$). The bias of the coin $\delta$ controls the correlation between $\vx$ and $\hat{\vx}$: $\delta = 0$ means $\vx$ and $\hat{\vx}$ are independent samples, while $\delta = 1$ means $\vx$ and $\hat{\vx}$ are independent draws from the same component. 
The definition ensures that the marginal of $\vx$ and $\hat{\vx}$ are equal to $F$ for any $\delta$.

% \paragraph{Augmentations} For a point $\vx$ drawn from a component $i \in [K]$, it's augmentation is: 
% \begin{align*}
%     \hat{\vx}\sim \mathcal{N}(\vmu_i,\vSigma)
% \end{align*}
\textbf{Spectral versus self-supervised methods.} Our results corroborate the known advantages of self-supervised methods over spectral methods. 
Clearly, spectral methods learn linear mappings directly on the data samples $\vx \sim \mathcal D$, whereas, 
% assuming as input a dataset of points, i.e., $\vx\sim \mathcal{D}$. 
% What do you mean by assuming a dataset of points? 
self-supervised methods operate on the augmented pairs $(\vx,\hat{\vx})\sim \mathcal{D}_{\mathrm{pair}}$. 
Consequently, their optimization problems differ in their target objectives:
\begin{equation*}
    \begin{aligned}
    \mathcal{L}_{\mathrm{spectral}}(\vA;\mathcal{D}) = {-||\vA^T\vx||^2} \qquad\&\qquad %\quad \mathrm{(explained\ variance/best\ fit)}\\
    \mathcal{L}_{\mathrm{selfsup}}(\vA;\mathcal{D}_{\mathrm{pair}}) = -\underbrace{(\vA^T\vx)^T\vA^T\aug{\vx}}_{\mathrm{attractive\ term}}
    +  \underbrace{\mathcal{R}(\vx)}_{\mathrm{regularizer}}
    \end{aligned}
\end{equation*}
While spectral objectives find the \emph{best-fit} subspace explaining maximum data variance, self-supervised objectives are based on the principle of bringing embeddings of similar point closer in conjunction with a loss-specific regularizer. The regularizer typically induces separation between distinct points.
% We elaborate on these differences in more detail in Sec~\ref{sec:contrastive_gmm}.
% Each method defines their minimization problem, which can be solved to give the optimal mapping corresponding to the method. 
% $$\vA_{\mathrm{spectral}} = \underset{\vA}{\mathrm{argmin}} \ \mathcal{L}_{\mathrm{spectral}}(\vA;\mathcal{D})$$
% $$\vA_{\mathrm{selfsup}} = \underset{\vA}{\mathrm{argmin}} \ \mathcal{L}_{\mathrm{selfsup}}(\vA;\mathcal{D}_{\mathrm{pair}})$$

\textbf{Linear dimensionality reduction (LDR).} Framing our task as LDR allows us to establish a measure of comparison between mappings learned by spectral ($\vA_{\mathrm{spectral}}$) and self-supervised methods ($\vA_{\mathrm{selfsup}}$).
% Consider points $\vx \in \mathbb{R}^d$ sampled from a Gaussian mixture $F$ in some high dimensional space $\mathbb{R}^d$. We want to project the points onto a $r$-dimensional space (where $r<<d$) with a projection matrix $\vA \in \mathbb{R}^{d\times r}$. 
% The goal of LDR is finding the optimal projection matrix.
For a SharedGMM, we denote it's "intra-component" variance (i.e. variance within a component) by $\vSigma$ and "inter-component" variance (i.e. the separation between components) by $\sum_k w_k \vmu_k \vmu^T_k$ \cite{fukunaga2013introduction}.
A GMM is said to be \textit{well-separated} if it has \emph{low intra-component} variance and \emph{high inter-component} variance.
LDR aims to map a GMM to a low dimensional subspace where it is well-separated. We evaluate the performance of these mappings by Fisher discriminant~\cite{fukunaga2013introduction}.  
% \pari{Need to elaborate on why fisher discriminant is the right thing to optimize. What is its formal relation to SNR ? what is for example the fisher discrimnarnt subspace for the parallel pancakes ?}
\begin{definition}[Fisher Discriminant]
\label{def:fisher_discriminant}    
Suppose a SharedGMM (Def~\ref{def:shared_gmm}) parameterized by $\{w_k,\vmu_k,\vSigma\}_{k\in[K]}$. 
% DEFINE HOW THE SHAREDGMM IS RELATED TO MATRIX A.
% Also, let $\vTheta \in \mathbb{R}^{d\times r}$ be an orthonormal basis for a $r$-dimensional subspace ${S}$. 
Then, the \textit{fisher discriminant} $J(\vA)$ for a mapping $\vA$ 
% $\vA \in \mathbb{R}^{d\times r}$ 
is defined as:  
\begin{align*}
    % J(\vA) &= \mathrm{Tr}([\vA^T \bar{\vSigma}\vA]^{-1}[\vA^T \bar{\vM}\vA])\\ 
    % &= \mathrm{Tr}\Big(\big[\vA^T (\sum_k w_k \vSigma_k) \vA\big]^{-1}\big[\vA^T (\sum_k w_k\vmu_k\vmu_k^T) \vA\big]\Big)
    J(\vA)= \mathrm{Tr}\left( \left[\vA^T \vSigma \vA \right]^{-1} \left[ \vA^T \left(\sum\nolimits_k w_k\vmu_k\vmu_k^T \right) \vA \right] \right)
\end{align*}
\end{definition} 
\begin{remark}
While $J(\vA)$ is defined in terms of $\vA$, its value is solely a function of $\col{\vA}$, i.e. the column space of $\vA$.
% (equivalently, subspace $\vA$ projects to). 
We hence use the mapping matrix $\vA$ and projection subspace $\col{\vA}$ interchangeably in our discussion. 
\end{remark}
% We hence also define the Fisher Discriminant for a subspace $S$ by $J_S$. Concretely, for a subspace $S$, let $\vA = [\vv_1,\vv_2,\ldots,\vv_r]$ be a basis of $S$ and $J(\vA)$ denote the fisher discriminant of $\vA$. Then 
% \begin{align*}
%     J_S(S) = J(\vA)
% \end{align*}
%Intuitively, $J(\vA)$ is maximized when the projection subspace captures all directions having the highest ratios of inter-component distance to intra-component variance. 
Note that $\vA^T \vSigma\vA$ and $\vA^T (\sum_k w_k\vmu_k\vmu_k^T)\vA$ denote the intra-component and the inter-component variances, respectively, \emph{after} projection via $\vA$. Thus, the Fisher discriminant measures the ratio of inter-component to intra-component variances post-projection. A \enquote{favorable} mapping $\vA$ will have a high value of the Fisher discriminant $J(\vA)$,  hence, the Fisher discriminant $J$ serves as the quantitative metric for evaluating the subspaces learned by different frameworks.

\textbf{Connection to linear discriminant analysis (LDA).} One can argue that it is not surprising for self-supervised methods to have better performance than classical spectral methods. The more interesting comparison would be against a \emph{supervised} method, which would help uncover the capabilities of contrastive methods. Therefore, we focus our attention on Linear Discriminant Analysis (LDA)~\cite{fukunaga2013introduction}. LDA is a \emph{supervised} LDR method that leverages the class information to learn the subspace where the data distribution conditioned on the class label (i.e. underlying component index) is maximally separated (see \cref{app:lda}). 
We highlight that assuming AeD (Def.~\ref{def:jointdistribution}) is strictly weaker than assuming a \emph{labeled} dataset. 
We show that this relaxed condition is enough for learning the Fisher subspace and gives similar performance to supervised LDR methods.
% $\vA_{\mathrm{spectral}}$ and $\vA_{\mathrm{selfsup}}$}. 
% To establish the utility of self-supervised methods it suffices to show $J(\vA_{\mathrm{selfsup}})>J(\vA_{\mathrm{spectral}})$. We instead show a stronger condition, informally, that $J(\vA_{\mathrm{selfsup}})$ is maximal, even for cases when $J(\vA_{\mathrm{spectral}})$ can fail abysmally (i.e., when $\vA_{\mathrm{spectral}}$ is worse than a random mapping). 

\section{Fisher subspace} \label{sec:fisher}

% \begin{table}[t]
% \centering
% \resizebox{0.8\textwidth}{!}{%
% \begin{tabular}{|l|c|c|}
% \hline
%  & \textbf{SVD Projection} & \textbf{Contrastive Learning} \\
% \hline
% \textbf{Data} & Single point $\vx \sim F$ & Pair of points $(\vx,\hat{\vx}) \sim \hat{F}$ \\
% \hline
% \multirow{ 2}{*}{\textbf{Objective}} & \multirow{ 2}{*}{Maximize data variance} & Min. intra-component variance\\
%  &  &  Max. inter-component variance \\
% \hline
% \textbf{Fisher-optimality}  & Spherical GMMs& Shared Covariance GMMs \\
% \hline
% \end{tabular}
% }
% \caption{\label{tab:annotator} We highlight the differences between classic SVD-based linear dimensionality reduction paradigm and our self-supervised (contrastive) learning-inspired problem setting.}
% % These settings in the data distribution they assume access to and the objective they optimize. Hence consequently they also in the classes of GMMs they find the fisher optimal subspace for.}
% % \vspace*{-15pt}
% \end{table} 
% % \ali{We should have a clear-cut statement. \enquote{We show empirically that ...}. We should also explain the important of learning without negatives in conjunction with the InfoNCE loss. We don't have a result that explicitly connects \cref{thm:gmm_simsiam} and \cref{thm:sharedcov_infonce}, so we need to have our story straight with respect to single modality results.}

% \subsection{Fisher Subspace}
For any target dimension $r$, the linear map that has the top $r$ eigenvectors of $\vSigma^{-1}(\sum_k w_k\vmu_k\vmu_k^T)$ as its columns maximizes the Fisher discriminant. 
Similarly, for any mappings $\vA_1$ and $\vA_2$, if $\col{\vA_1} \subseteq \col{\vA_2}$, then we have $J(\vA_1) \leq J(\vA_2)$. In other words, $J(\cdot)$ is monotonic in the subspace of the matrices. Based on this, we will have the following definition, which is necessary in \emph{measuring the optimality of projections} learned by different methods.
%Since the dimension of column space of $\bar{\vSigma}^{-1}\bar{\vM}$ is $\leq K$, we can define a sufficient projection subspace of dimension $\leq K$, which maximizes the fisher discriminant. This subspace is termed as fisher subspace~\cite{brubaker2008isotropic,fukunaga2013introduction} and is defined as :
\begin{definition}[Fisher Subspace] Given a SharedGMM (Def~\ref{def:shared_gmm}) parameterized by $\{w_k,\vmu_k,\vSigma\}_{k\in[K]}$, its \textit{Fisher subspace}~\cite{fukunaga2013introduction}, denoted by ${S}_{F}$, is the smallest subspace that achieves the maximum Fisher discriminant, which is given by:
\begin{align}\label{eqn:sfexplicit}
     % S_F = \Span{\{\vSigma^{-1}\vmu_1,\vSigma^{-1}\vmu_2,\ldots,\vSigma^{-1}\vmu_K\}}
     S_F = \Span{\{\vSigma^{-1}\vmu_k\}_{k\in[K]}}
\end{align}
\end{definition}
% Hence the fisher subspace ~\cite{brubaker2008isotropic,fukunaga2013introduction} is the minimal optimal subspace for a given SharedGMM.
We also use a unique property of the Fisher subspace~\cite{fukunaga2013introduction}; it is the smallest subspace preserving the class posterior probabilities for Gaussian mixtures with shared covariance, which we formalize next.
\begin{lemma}
Let $\{w_k,\vmu_k,\vSigma\}_{k\in[K]}$ be a SharedGMM and $\Pr(z=k|\vx)$ be the posterior probability of $\vx$ being drawn from the component $z$. Let $S_F$ be the mixture's Fisher subspace and $\vA_F$ be a projection matrix such that $\col{\vA_F} = S_F$. Then for any $\vx \sim F$,
\begin{align*}
    \Pr(z=k|\vx) = \Pr(z=k|\vA^{T}_F\vx)\  \forall k\in [K]
\end{align*}
\end{lemma}
{Crucially, this property implies that projecting a GMM onto its Fisher subspace will not result in mode collapse or erroneous mode merging. Moreover, any clustering algorithm operating in the lower dimensional subspace will observe the same class probabilities as in the original space.
% WHAT DO YOU MEAN BY THAT? PLEASE SIMPLIFY :)
Owing to these properties of the Fisher subspace, 
% it's properties of maximizing the ratio of inter to intra-component variances and preservation of class probabilities (i.e., index of underlying component), 
we argue that it is \emph{the optimal subspace for projection}, particularly for the class of shared covariance GMMs. Under the fully labeled SharedGMM setting, one could deduce that the subspace learned by (multi-class) LDA, $S_{LDA}$, for SharedGMM coincides with Fisher subspace $S_F$.}

\subsection{A simple example: Parallel pancakes~\cite{kannan2009spectral}}\label{subsec:parallel}
\begin{figure}[t]
\centering
% \hspace*{0.5cm}
% \includegraphics[width=0.6\textwidth]{images/fisherspace_v2.pdf}
\includegraphics[width=0.65\textwidth]{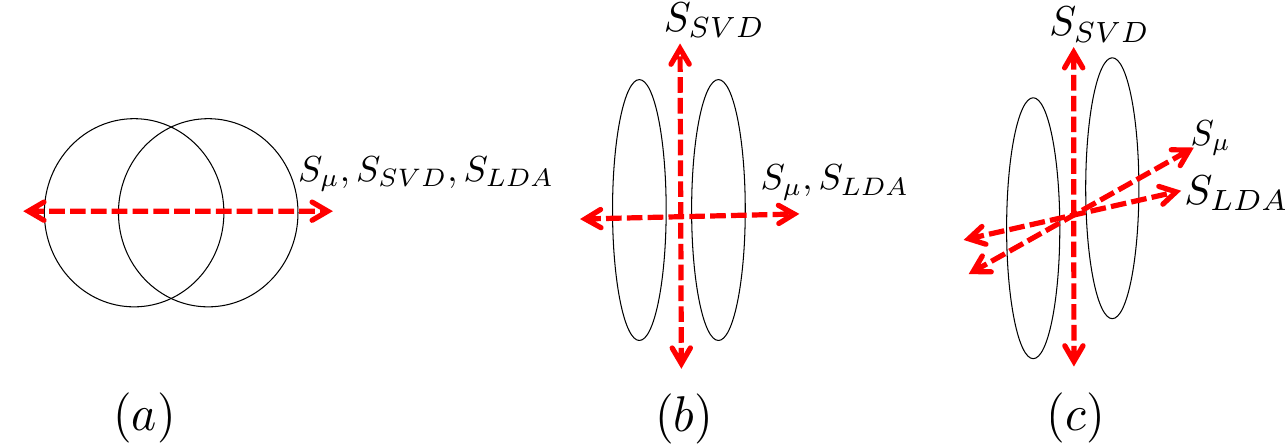}
% \hspace*{-0.5cm}
\caption{(a) (Spherical Gaussians) $S_{SVD} = S_{LDA}$ and projection onto $S_{SVD}$ leads to \emph{well seperated} GMM. (b) (Non-spherical Gaussians) Large variance in one direction means $S_{SVD} \neq S_{LDA}$ and projection onto $S_{SVD}$ leads to mode collapse. (c) (Shifted non-spherical Gaussians) $S_{\vmu} \neq S_{LDA}$.}
\label{fig:fig1}
\vspace*{-15pt}
\end{figure}
% \begin{definition}
%     The $r$-dimensional \emph{SVD subspace}~\cite{vempala2004spectral} $S^{(r)}_{SVD}$ for a GMM $F$ is described as the top $r$ left singular vectors of $\mathbb{E}_{\vx \sim F}[\vx\vx^T]$.
% \end{definition}
% \textbf{Spectral Dimensionality Reduction : SVD-subspace.} 
% Having characterized the fisher-optimal low dimensional space, 
% We present our  classical dimensionality reduction methods which aim to learn the projection matrix corresponding to the fisher subspace. 
Spectral methods~\cite{kannan2009spectral} are the standard methods for LDR. They are based on the principle of finding the best-fit subspace which 
% maximally explains the data. This subspace 
is (generally) given by the \emph{top singular vectors of data covariance}. Formally, the $r$-dimensional \emph{SVD subspace}~\cite{vempala2004spectral} $S^{(r)}_{SVD}$ for a GMM $F$ is described as the top $r$ left singular vectors of $\mathbb{E}_{\vx \sim F}[\vx\vx^T]$. 
% They use spectral information (i.e., the singular values) of the data covariance matrix and hence are named so. 
A vanilla spectral dimensionality reduction method hence projects the data points onto the $r$-dimensional SVD subspace.
% (in other words, the mapping matrix $\vA$ has the orthogonal basis of the SVD subspace as its columns).
% Consider \textit{Spherical Gaussian Mixture Model} (SphereGMM) as a special case of SharedGMM, with $\vSigma = \vI$. 
~\cite{kannan2009spectral} proved that for \textit{spherical} GMMs (i.e. $\vSigma = \vI$), SVD subspace is equal to the mean subspace $S_{\vmu} = \Span{\vmu}$. Moreover, using Eq.~\eqref{eqn:sfexplicit} we can derive that $S_{LDA} = S_{\vmu} = S_{SVD}$. 
% we have the following result.
% \begin{theorem} (from)
%      The $K$-dim SVD subspace for a mixture of $K$ Spherical Gaussians  contains the $\Span{\{\vmu_1,\vmu_2,\ldots,\vmu_K\}}$
% \end{theorem}
Hence, projection onto the SVD subspace is \emph{the optimal} dimensionality reduction method for \emph{spherical} GMMs.

However, it is easy to construct examples demanding for a sophisticated method. Spectral methods aim to maximally explain the data, leading to possibly \enquote{bad} projections when the variance in certain directions dominates the separation between the means.
% , for which we introduce the \enquote{parallel pancakes} example ~\cite{kannan2009spectral}.
% \draft{To highlight the importance of our results, and how they are meaningfully better than spectral methods, we now describe the well known "Parallel Pancakes" example~\cite{kannan2009spectral}.} 
In \cref{fig:fig1} we present such an example; consider a two component GMM that resembles “parallel pancakes” such that the components are narrow and separated along one direction, and spherical in all other directions. The two dimensional SVD subspace for this GMM is given by a plane parallel to the pancakes. The Fisher subspace (and equivalently LDA subspace), however, is the plane \emph{passing through the means}; formally, $S_{\mathrm{SVD}} \not\subset S_F$. In fact, the $S_{SVD}$ has the \emph{smallest} Fisher discriminant among all two-dimensional subspaces.
% , and hence, is worse than projection onto a random plane. 
Finding the optimal subspace for non-spherical GMMs is known to be non-trivial~\cite{brubaker2008isotropic,kannan2009spectral}. 
% Our results show that self-supervised learning can naturally deal with these cases and hence goes meaningfully beyond spectral methods. 

\section{Contrastive learning for Gaussian mixture models}\label{sec:contrastive_gmm}

We analyze solving the dimensionality reduction task for GMMs with two popular self-supervised methods, SimCLR and SimSiam. Both methods build on using augmentation pairs, but differ in their optimization objective (specifically how they regularize). We show that both objectives are able to effectively leverage the augmentations and learn mappings onto the (optimal) Fisher subspace. 

\subsection{Optimization objective}\label{subsection:optimization_contrastive}
\textbf{InfoNCE loss~\cite{oord2018representation}} is popularly used in contrastive learning methods (like SimCLR). It learns representations that pull data points and their augmentations close together while pushing them away from other points in the embedding space. We define the InfoNCE objective for linear mappings as,
\begin{align} \label{eqn:infonce}
\begin{split}
% \mathcal{L}(\vA) = & -\underset{\vx\sim F}{\mathbb{E}}\Bigg[\underset{\aug{\vx}\sim\mathcal{T}(\vx)}{\mathbb{E}}\bigg[(\vA^T\vx)^T\vA^T\aug{\vx}\bigg]\Bigg] \\ 
% & + \underset{\vx\sim F}{\mathbb{E}} \Bigg[\log\bigg(\underset{\negative{\vx} \sim F}{\mathbb{E}}\bigg[\exp\big((\vA^T\vx)^T\vA^T\negative{\vx}\big)\bigg]\bigg)\Bigg]
& \mathcal{L}_{\mathrm{Info}}(\vA) =  -\underset{(\vx,\aug{\vx}) \sim \hat{F}}{\mathbb{E}}\Bigg[(\vA^T\vx)^T\vA^T\aug{\vx}\Bigg] + \underset{\vx\sim F}{\mathbb{E}} \Bigg[\log\bigg(\underset{\negative{\vx} \sim F}{\mathbb{E}}\bigg[\exp\big((\vA^T\vx)^T\vA^T\negative{\vx}\big)\bigg]\bigg)\Bigg].
\end{split}
\end{align}
where $(\vx,\hat{\vx}) \sim \hat{F}$ is a sample from augmentation-enabled distirbution (Def~\ref{def:jointdistribution}) such that $\vx\sim F$ and $\negative{\vx}\sim F$ are independent draws from the mixture. 
% (equivalently $(\vx,\cdot) \sim \hat{F}$). 
% \draft{We consider the population loss.} 
% CAN WE REMOVE THIS SENTENCE, OR SHOULD IT STAY?
% $\col{\vA}$ (i.e. column space of $\vA$) is the lower-dimensional subspace we are mapping to. 
The attractive term (i.e. the first term) keeps $\vx$ and $\hat{\vx}$ close by maximizing their dot product, while regularization term penalizes {proximity} to random samples. In other terms, the first term increases inter-component variance 
% and promote separation of the components 
while the second term decreases data variance and hence intra-component variance. 

\textbf{SimSiam loss~\cite{chen2021exploring}} is another popular self-supervised loss without the negatives. 
%\begin{align*}%\label{eqn:simsiam_loss}
%    \underset{\begin{subarray}{c}
%    \vA,\vA_p \in \mathbb{R}^{d\times r}\\
%    \vA_p \in \mathbb{R}^{r\times r}\\
%    ||\vA||_2\leq 1\\
%    ||\vA_p||_2\leq 1\\
%    \end{subarray}}{\mathrm{argmin}} \underset{(\vx,\aug{\vx}) \sim \hat{F}}{-\ \mathbb{E}}\Bigg[\bigg(\frac{\vA_p^T\vA^T\vx}{||\vA_p^T\vA^T\vx||_2}\bigg)^T\mathrm{StopGrad}\bigg(\frac{\vA^T\aug{\vx}}{||\vA^T\aug{\vx}||_2}\bigg)\Bigg]
%\end{align*}
\begin{align*}%\label{eqn:simsiam_loss}
    \underset{(\vx,\aug{\vx}) \sim \hat{F}}{-\ \mathbb{E}}\Bigg[\bigg(\frac{\vA_p^T\vA^T\vx}{||\vA_p^T\vA^T\vx||_2}\bigg)^T\mathrm{StopGrad}\bigg(\frac{\vA^T\aug{\vx}}{||\vA^T\aug{\vx}||_2}\bigg)\Bigg]
\end{align*}
{Unlike InfoNCE, Simsiam only operates on augmentations (from $\hat{F}$) and doesn't use negatives. The loss is parameterized by two matrices; the mapping $\vA$ and the ``prediction head'' $\vA_p$, which is not utilized for projection. 
% The StopGrad operation zeros out the gradients flowing through it, whose main role is breaking the symmetry between encoders in Siamese networks~\citep{wang2022importance}.
Observe that Simsiam might be prone to a collapsing solution, i.e., mapping all points onto the same vector. Prior work~\cite{tian2021understanding,jing2021understanding,wen2022mechanism} has argued that $\vA_p$, which makes the optimization asymmetric, and StopGrad, which zeros out the gradients flowing through it, play a crucial role together in the guiding the training dynamics, preventing the occurrence of collapse.}
Our goal is to analyze the fixed point of this loss function. For tractability, we introduce some simplifying assumptions and examine a modified loss:
\begin{align}\label{eqn:simsiam_loss}
    \mathcal{L}_{\mathrm{Siam}}(\vA) = & -\mathbb{E}_{(\vx,\aug{\vx}) \sim \hat{F}}\left[(\vA^T\vx)^T\vA^T\aug{\vx}\right] + \xi\ \mathbb{E}_{\vx\sim F} \left[||\vA^T\vx||^2\right]
\end{align}
We remove the StopGrad operation and set the prediction head (i.e. $\vA_p$) to $\vI$. 
We also trade the normalization term with a regularization term weighted by $\xi$. The attractive term in SimSiam behaves similar to InfoNCE loss, while the regularization term penalizes the norm of the projected points. The loss scales linearly with $\| \vA \|^2$, therefore, we impose a norm constraint when optimizing for $\vA$. 
% \ali{We need to explain what SimSiam loss is better. For instance, it is not totally clear what StopGrad does and why it is there. Unknowledgeable readers will have hard time understanding the use of SimSiam loss.}
% \ali{Present simplified SimSiam setup clearly, maybe with bullet points.}

{We have already established that maximizing total variance \emph{on its own} in the projected space is not the preferable strategy for dimensionality reduction. In the presence of augmentation (and negatives), self-supervised objectives could take a finer-grained approach by maximizing inter-component variance and keeping intra-component variance low, simultaneously.
% spectral methods (specifically, SVD-subspace projection) learn projection matrices that induces maximal variance in the projected space; 
% Spectral methods (specifically, SVD-subspace projection) learn projection matrix such that the projected points are the best-fit for data (or equivalently, we have maximal variance in projected space). Spectral methods are based on maximization of : 
% \begin{align}\label{eqn:svdobjective}
%     \underset{\{\vA\in\mathbb{R}^{d\times r}:\vA^T\vA = I\}}{\mathrm{argmax}} \underset{\vx \sim F}{\mathbb{E}}[||\vA^T\vx||^2]
% \end{align}
% Informally, this objective is orthogonal to the goal of dimensionality reduction. Recall that our dimensionality reduction definition aims to find a projection space such that intra-component variance is \emph{low} while inter-component variance is \emph{high}. 
% We can observe that optimizing a self-supervised objective serves as a better proxy for dimensionality reduction. 
Next, we will formalize our claim that self-supervised learning is a good proxy for fully supervised techniques (e.g., LDA).} 

\subsection{Main theorem}
% \ali{Let's bring the results for single and multi-modal settings under the same section. This looks disjoint and difficult to follow.}

Nex, we prove that InfoNCE and (simplified) Simsiam can find the Fisher-optimal projection matrix for the class of SharedGMMs.

% We first show 
% that minimizing the InfoNCE loss  learns the Fisher LDA subspace when the covariances are shared across the mixture components. We first present a proposition that 
% minimizing InfoNCE loss for spherical GMMs learns the fisher subspace (Prop~\ref{prop:isotropic_infonce} and then extend the proposition to shared gaussian mixtures with shared covariances (Thm~\ref{thm:sharedcov_infonce})

% \pari{The solution to the InfoNCE loss mostly likely has an unbounded norm. The proof still works if I put the condition that $||\vA||_2<\Delta$. Should I do that?}
% Given Prop~\ref{prop:isotropic_infonce}, we prove our main result, i.e.
% For a distribution $F$ belonging to the class of shared covariance gaussian mixtures, if we have access to it's augmentation-enabled distribution $\hat{F}$ (Def~\ref{def:jointdistribution}), then minimizing the InfoNCE loss learns a projection matrix $\vA$, such that $\col{\vA}$ is a subset of the fisher subspace $S^{(K)}_F$. 

\begin{theorem} \label{thm:sharedcov_infonce}
Suppose $F$ is a SharedGMM parameterized by $\{w_k,\vmu_k,\vSigma\}_{k\in[K]}$ and $\hat{F}$ is its \jointdistribution with bias $\delta$. 
Let $S_F$ be the Fisher subspace (Eqn~\ref{eqn:sfexplicit}) of $F$. 
% and define $\mathcal{L}_{\mathrm{Info}}$ and $\mathcal{L}_{\mathrm{Siam}}$ to be the InfoNCE loss (Eqn~\ref{eqn:infonce}) and SimSiam loss (Eqn~\ref{eqn:simsiam_loss}) respectively. 
Denote $\lambda_{\mathrm{min}}$ as the minimum non-zero eigenvalue of $\,\vSigma^{-\frac{1}{2}} \left( \sum_k w_k\vmu_k\vmu_k^T \right) \vSigma^{-\frac{1}{2}} $ and
\begin{align*}
    \vA_{\mathrm{Info}} = \underset{\vA \in \mathbb{R}^{d\times r}}{\mathrm{argmin}}\  \mathcal{L}_{\mathrm{Info}}(\vA), \quad\qquad \vA_{\mathrm{Siam}} = \underset{\begin{subarray}{c}
    \vA \in \mathbb{R}^{d\times r},\ ||\vA||_2\leq 1\\
    \end{subarray}}{\mathrm{argmin}}\   \mathcal{L}_{\mathrm{Siam}}(\vA)
\end{align*}
% Then given $r\geq K$, $\col{\vA^*} \subseteq \Span{\{\vmu_k\}_{k\in[K]}}$. Moreover if $\delta = 1$, then $\col{\vA^*} = \Span{\{\vmu_k\}_{k\in[K]}}$.
Then for some $r\geq K$, $S_{\mathrm{Info}} \triangleq \col{\vA_{\mathrm{Info}}}$, $S_{\mathrm{Siam}} \triangleq \col{\vA_{\mathrm{Siam}}}$, we have : 
\begin{itemize}
    \item $S_{\mathrm{Info}} \subseteq S_F$. If $\delta = 1$, then $S_{\mathrm{Info}} = S_F$
    \item $S_{\mathrm{Siam}} \subseteq S_F$. If $0<\xi<\frac{\delta \lambda_{\mathrm{min}}}{1+\lambda_{\mathrm{min}}}$, then $S_{\mathrm{Siam}} = S_F$
\end{itemize}

\end{theorem}
\begin{remark}
Access to \jointdistribution 
% (with other mild conditions) 
is \textbf{provably sufficient} for learning the Fisher subspace for SharedGMMs matching the performance of supervised LDA. The result implies that class labels which are typically assumed by supervised LDR methods like LDA, are not necessary.
\end{remark}

The proof of \cref{thm:sharedcov_infonce} can be found in \cref{app:proofofmain,app:simsiamproof}. 
% Note that while $\vA^*$ is not a projection matrix, it can be orthonormalized into one. Hence we use the term projection matrix for $\vA^*$ in our discussion. 
For population setting, the fact that self-supervised methods learn mapping onto a subset of the Fisher space implies that the projected points do not capture any \textit{noise} and only contain (a subset of) the useful \textit{features}. Note that spectral methods do not have guarantees of this form. The theorem also states if we have perfect augmentations for InfoNCE loss, i.e. $\delta = 1$, we cover \emph{all directions} in the Fisher space and learn it exactly. The equivalence is true for SimSiam loss when the regularization coefficient $\xi$ is upper bounded appropriately. On a related front, this precludes learning of the collapsing solution for SimSiam loss, verifying the claims in prior work on a simplified version of the loss function. 
Most importantly, self-supervised methods in question learn \emph{the same subspace} as \emph{supervised} dimensionality reduction methods like LDA which needs the knowledge of underlying components for all the samples. 
% \ali{let's briefly mention technical difficulties and that we leave this important open direction for future work.}

\textbf{Technical difficulties.} Although we do not provide a proof for $S_{\mathrm{Info}} = S_F$ when $0<\delta<1$, we conjecture that it is true, for which we include an empirical discussion in \cref{sec:exp}. \cref{fig:delta_plot} suggests that $S_{\mathrm{Info}} = S_F$ even for $\delta <1$, analysis of which we leave as future work. Additionally, we do not have an exact characterization of InfoNCE solution; we characterize just the column space of $\vA_{\mathrm{Info}}$. 
Our empirical results show that the mapped points are well-separated compared to a projection onto the Fisher subspace (see Fig~\ref{fig:scaling_plot}).
Furthermore, while we prove that both contrastive and non-contrastive objectives learn the same subspace, further investigation is required for their direct comparison. In \cref{sec:simsiam_eg}, we provide an example where $\vA_{\mathrm{Info}}$ is strictly better than $\vA_{\mathrm{SimSiam}}$ (for $r<K$). 
We also present an empirical study on the effect of $r$ (see Figure~\ref{fig:rank_plot}).

\section{Multi-modal Gaussian mixture models} \label{sec:multimodal}

% \pari{can use parts of this in the intro}
Previously, we considered a setup where a point and its augmentation follow the same distribution. It is a natural assumption for methods like SimCLR and SimSiam where augmentations of points (i.e. images) are defined by transformations (e.g., cropping, color jittering) on the image. This assumption may not hold in general. Text embeddings models like DPR~\cite{karpukhin2020dense} or image-text embedding models like CLIP~\cite{radford2021learning} consider input samples as a pair of points following (possibly) different distribution. 

For instance, each sample could be a pair of an image with its corresponding caption (for CLIP) or a search query with a relevant document that answers the query (for DPR). The objective in this \emph{multi-modal} setup is to learn a joint representation space for both modalities. 
% such that representations of points in a pair are close, while being far from random. 
These joint representations can be used downstream for finding similarities between data points belonging to different ambient spaces, and are surprisingly competitive with fully supervised representations~\cite{radford2021learning}. Our goal is to theoretically analyze the representations learned by such model, specifically the CLIP model. Next, we define CLIP-GMM to capture multi-modal data in a theory-friendly setting.  
\begin{definition}(CLIP-GMM)
\label{def:clip_gmm} 
A \emph{CLIP Gaussian mixture} (CLIP-GMM) is defined as the probability distribution  
% \begin{align*}
%     F_{\mathrm{clip}} = \sum_{k \in [K]} w_k \mathcal{N}(\vmu_{V,k},\vSigma_V) \times \mathcal{N}(\vmu_{T,k},\vSigma_T)
% \end{align*}
$F_{\mathrm{clip}} = \sum_{k \in [K]} w_k \mathcal{N}(\vmu_{V,k},\vSigma_V) \times \mathcal{N}(\vmu_{T,k},\vSigma_T)$
where $w_k$ are the mixture weights, $\{\vmu_{V,k}\}_{k\in [K]},\vSigma_V$ and $\{\vmu_{T,k}\}_{k\in [K]},\vSigma_T$ are the parameters for the two coordinate spaces.
% respectively.
\end{definition}
CLIP-GMM is a mixture of product distribution over Gaussians. To elaborate, sampling a pair $(\vx_V,\vx_T) \sim F_{\mathrm{clip}}$ is a two step process. We first sample an underlying component index, and then draw independent samples from the component \emph{with the same index} in the respective coordinate space. 
$F_{\mathrm{clip}}$ can be used to define marginal distribution over each coordinate space, for instance, 
$F_V = \sum_{k \in [K]} w_k \mathcal{N}(\vmu_{V,k},\vSigma_V)$.
% \begin{align*}
%     F_V & = \sum_{k \in [K]} w_k \mathcal{N}(\vmu_{V,k},\vSigma_V)
% \end{align*}

% We want to learn \enquote{good} joint representations multi-modal data.  

\subsection{Optimization objective} 
Following CLIP~\cite{radford2021learning}, we want to learn different representation functions for each modality. We let these functions be linear mappings, i.e., $\vA_V \in \mathbb{R}^{d_1 \times r},\vA_T \in \mathbb{R}^{d_2 \times r}$. Recall that optimal mapping matrices for each coordinate space would have their column space equal to that of Fisher subspace for the marginal distributions $F_V,F_T$. Concretely, let the representation of a sample $(\vx_V,\vx_T)$ be given by $(\vA^T_V\vx_V,\vA^T_T\vx_T)$, where $\vA_V,\vA_T$ are the mappings. Then, we define CLIP InfoNCE loss~\cite{radford2021learning} as,
\begin{align*} %\label{eqn:infonce_clip}
% \begin{split}
% {\displaystyle \mathcal{L}_{\mathrm{clip}}(\vA_V,\vA_T) =  -\underset{(\vx_V,\vx_T)\sim F_{\mathrm{clip}}}{\mathbb{E}}\bigg[(\vA^T_T\vx_T)^T\vA^T_V\vx_V\bigg]} \notag \\ 
% {\displaystyle +\underset{\vx_T\sim F_T}{\mathbb{E}} \bigg[\log\bigg(\underset{\vx_V \sim F_V}{\mathbb{E}}\big[\exp\big((\vA^T_T\vx_T)^T\vA^T_V\vx_V\big)\big]\bigg)\bigg]}
% (\vA_V,\vA_T) 
 \mathcal{L}_{\mathrm{clip}} = \underset{(\vx_V,\vx_T)\sim F_{\mathrm{clip}}}{\mathbb{-E}} \left[(\vA^T_T\vx_T)^T\vA^T_V\vx_V \right] 
+\underset{\vx_T\sim F_T}{\mathbb{E}} \Big[\log \Big(\underset{\vx_V \sim F_V} {\mathbb{E}} \left[ \exp \left( (\vA^T_T\vx_T)^T\vA^T_V\vx_V \right) \right] \Big) \Big]
\end{align*}
Each coordinate space serves as augmentation for the other. Similar to InfoNCE, CLIP InfoNCE attracts embeddings of $\vx_V$ and $\vx_T$ via the first term, while regularizing with the Log-Sum-Exp term.

\subsection{Results}

Our main result for CLIP-GMM states that we can learn a subset of Fisher subspace for the constituent modalities by minimizing the CLIP InfoNCE loss. We state our results in the following theorem.  
\begin{theorem} \label{thm:clipgmm_infonce}
Suppose $\{w_k,\vmu_{V,k},\vmu_{T,k},\vSigma_V,\vSigma_T\}_{k\in[K]}$ is a CLIP-GMM. 
Let the Fisher subspace of $F_V$ be $S_{V,F}$ and $F_T$ be $S_{T,F}$. 
Denote $\vA^*_V,\vA^*_T$ as the optimal solution of the CLIP InfoNCE loss,
\begin{align*}
    % \vA^*_V,\vA^*_T = \underset{\vA_V,\vA_T \in \mathbb{R}^{d\times r}}{\mathrm{argmin}}\  \mathcal{L}_{clip}(\vA_V,\vA_T) 
    \vA^*_V,\vA^*_T = \underset{\begin{subarray}{c}
    \vA_V \in \mathbb{R}^{d_1 \times r},\ \vA_T \in \mathbb{R}^{d_2 \times r}
    \end{subarray}}{\mathrm{argmin}}\  \mathcal{L}_{\mathrm{clip}}(\vA_V,\vA_T) 
\end{align*}
% Then given $r\geq K$, $\col{\vA^*} \subseteq \Span{\{\vmu_k\}_{k\in[K]}}$. Moreover if $\delta = 1$, then $\col{\vA^*} = \Span{\{\vmu_k\}_{k\in[K]}}$.
For any $r\geq K$, let $S_{V,\mathrm{clip}} = \col{\vA^*_V}$ and $S_{T,\mathrm{clip}} = \col{\vA^*_T}$. Then, $S_{V,\mathrm{clip}} \subseteq S_{V,F}$ and $S_{T,\mathrm{clip}} \subseteq S_{T,F}$.
\end{theorem}
The proof of \cref{thm:clipgmm_infonce} can be found in \cref{app:clipproof}. 
In simple words, CLIP InfoNCE learns the subset of the Fisher subspace instead of the \emph{exact space}, which is \emph{weaker} than the single-modal setting. This is due to the fact that the means in the two spaces can vary arbitrarily. Hence, one can choose means and covariances adversarially such that particular directions in the Fisher subspace of either $F_V$ or $F_T$ does not contribute to the inter-component distance and are not learned by the CLIP InfoNCE. For certain special configuration of model parameters, such as $\vSigma_T^{-\frac{1}{2}}\vmu_{T,k} = \vSigma_V^{-\frac{1}{2}}\vmu_{V,k}$, we can achieve $S_{T,\mathrm{clip}} = S_{T,F}$ and $S_{V,\mathrm{clip}} = S_{V,F}$. {The result still shows that self-supervised learning for multi-modal data is better than non-supervised methods. Accessing augmentations are strictly weaker than the knowledge of the labels, and they also naturally occur in the multi-modal setting. The fact that contrastive losses learn a subset of the optimal subspace verifies their capabilities.}

\begin{figure*}[t]
\begin{subfigure}{0.495\linewidth}
\includegraphics[width=0.48\linewidth]{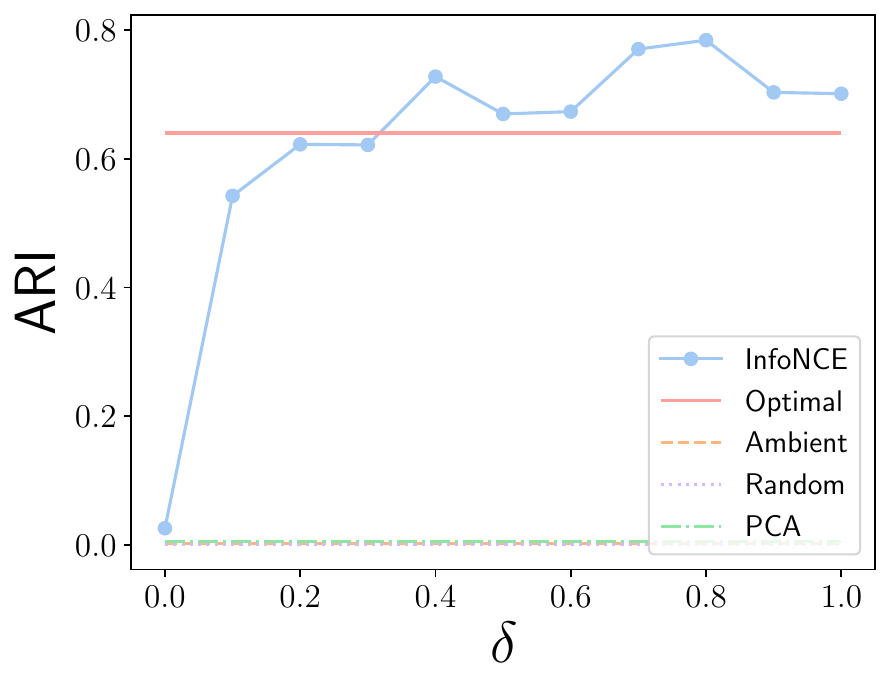}
\includegraphics[width=0.48\linewidth]{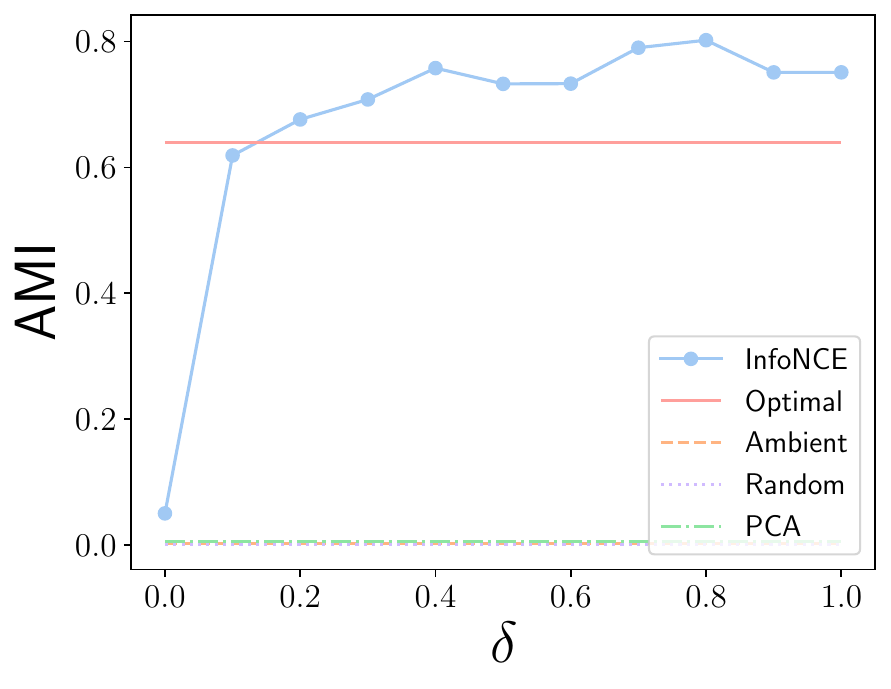}
\caption{Noise in augmentation}
\label{fig:delta_plot}
\end{subfigure}\hfill
\begin{subfigure}{0.495\linewidth}
\includegraphics[width=0.48\linewidth]{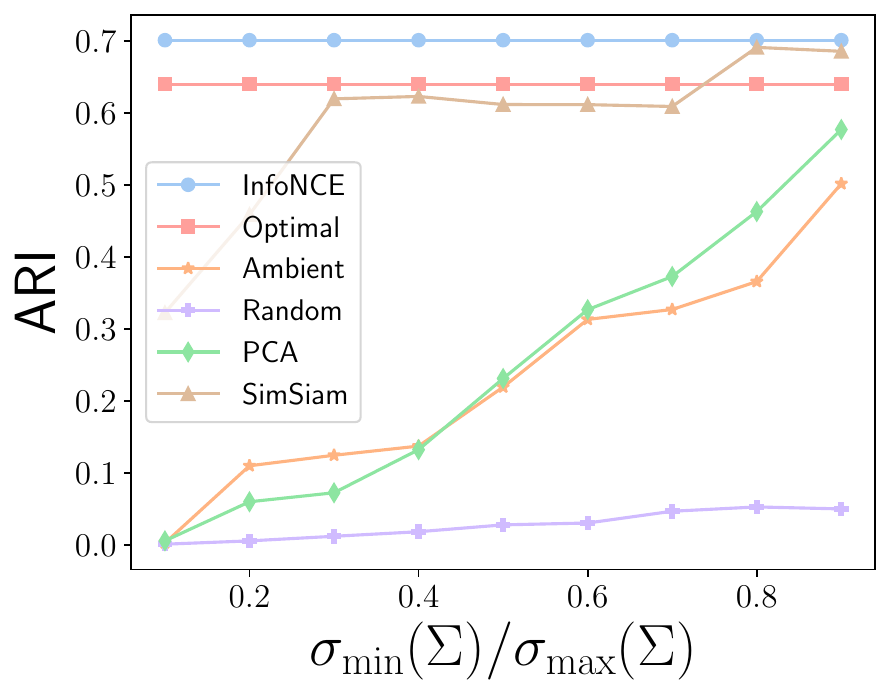}
\includegraphics[width=0.48\linewidth]{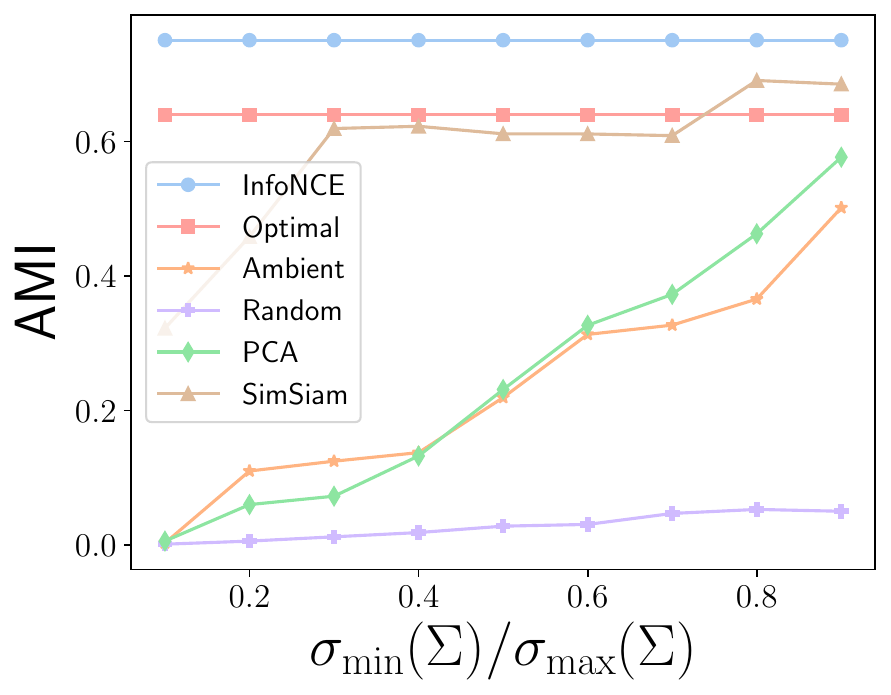}
\caption{Variance $\perp S_F$}
\label{fig:condition_plot}
\end{subfigure}

\medskip

\begin{subfigure}{0.495\linewidth}
\includegraphics[width=0.48\linewidth]{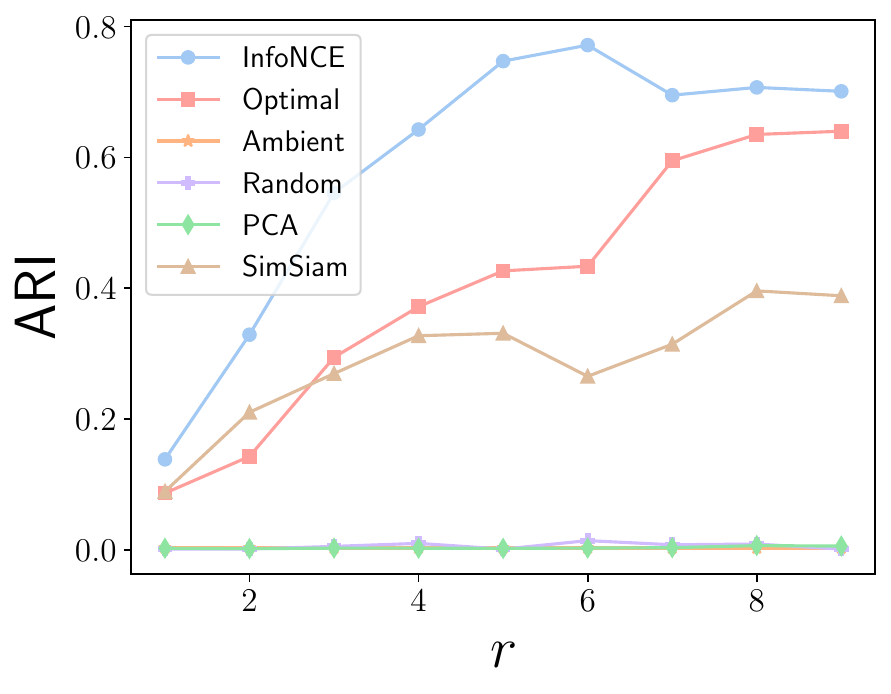}
\includegraphics[width=0.48\linewidth]{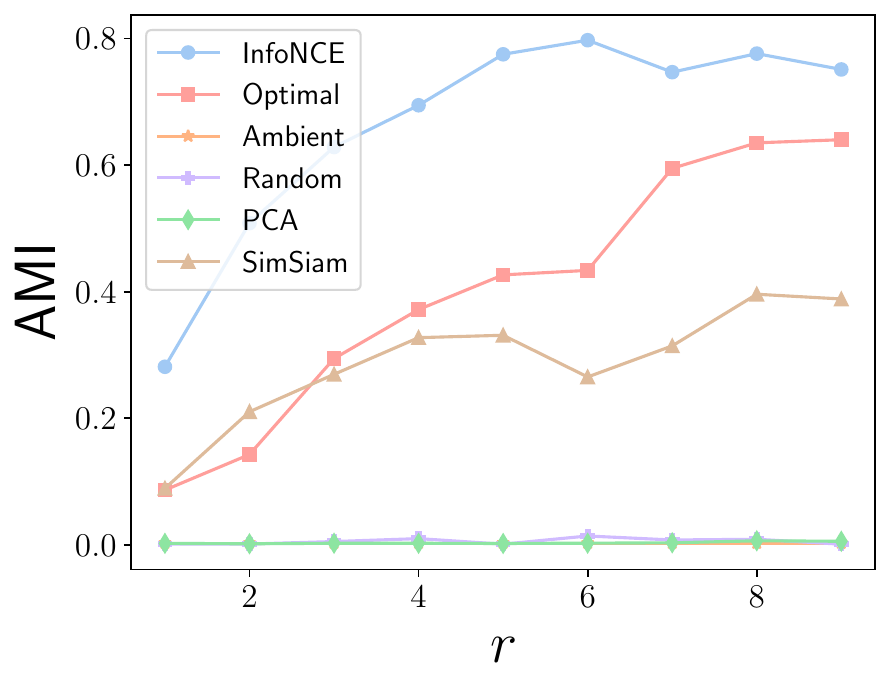}
\caption{Dimension of projected space}
\label{fig:rank_plot}
\end{subfigure}\hfill
\begin{subfigure}{0.495\linewidth}
\includegraphics[width=0.48\linewidth]{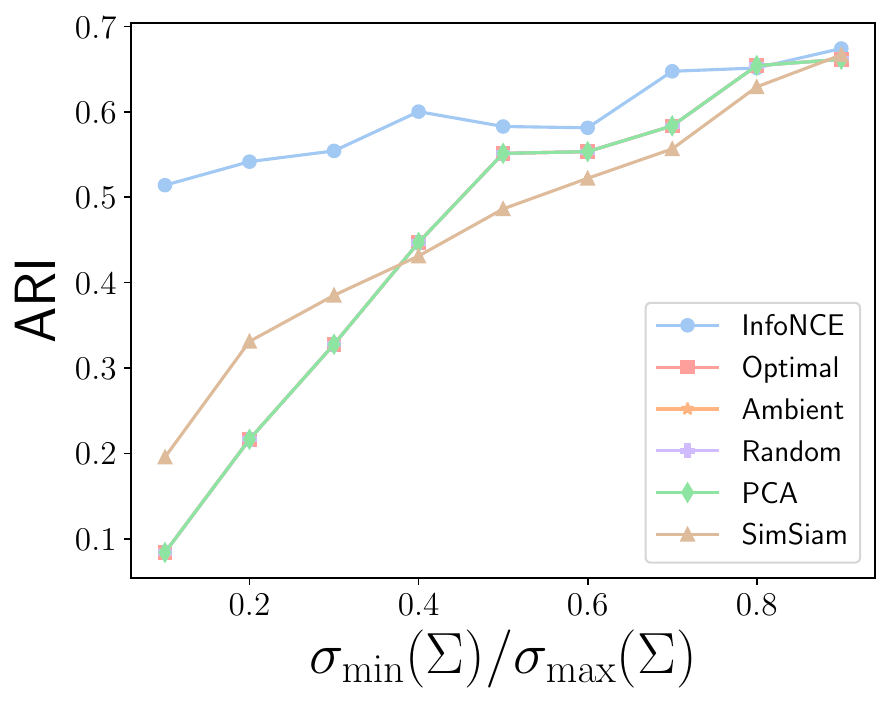}
\includegraphics[width=0.48\linewidth]{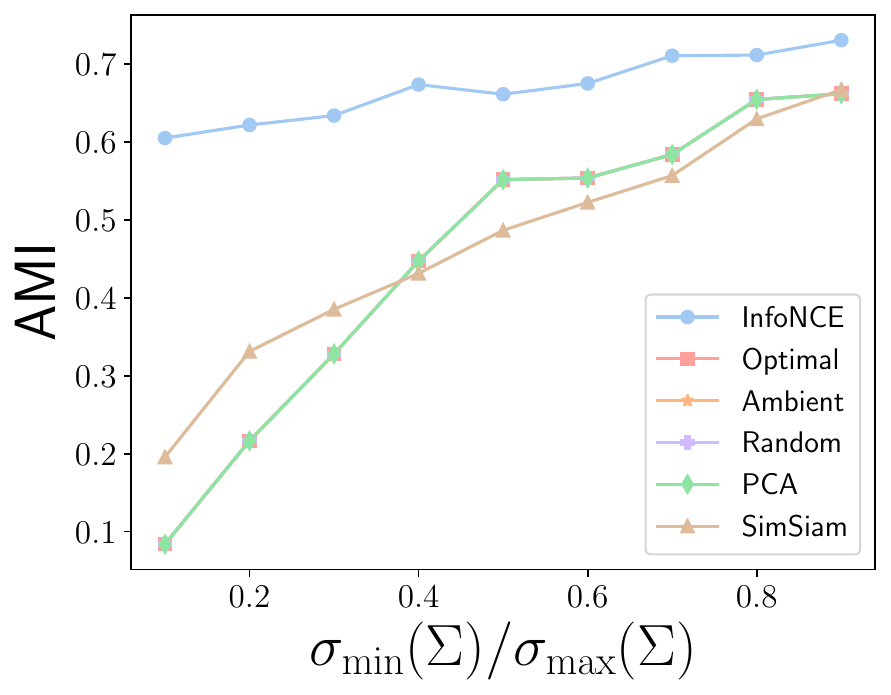}
\caption{Variance within $S_F$}
\label{fig:scaling_plot}
\end{subfigure}
\caption{\label{fig:mi_plots} We empirically validate our theoretical findings for four main settings. (a) InfoNCE is robust to noise in augmentations for various values of $\delta$. (b) InfoNCE (and SimSiam) are invariant to variance \emph{orthogonal} to Fisher subspace. (c) InfoNCE outperforms spectral methods for every mapping dimension. (d) InfoNCE learns a good scaling within the Fisher subspace}
\end{figure*}

\section{Experiments}\label{sec:exp}
% \pari{do we want plots with varying 
%\begin{itemize}
%    \item number of components
%    \item number of samples per component
%    \item varying ambient dimension
%\end{itemize}
%Don't believe that these are that interesting right?
%}

We validate our theoretical findings with experiments on synthetic and real data. For the synthetic data, we study the effect of noise $\delta$ in \jointdistribution, rank $r$ of the projection matrix and condition number of covariance matrix on learned representations. For the real data experiments, we compare self-supervised methods against baselines for clustering CIFAR100 dataset on low dimensional subspaces using K-Means.

\subsection{Setup}
We evaluate different linear dimensionality reduction methods for SharedGMMs. We adopt a two-step process; the first step is dimensionality reduction with the target method, and the second step is clustering with an out-of-the-box clustering algorithm (i.e. K-Means
% \footnote{\url{https://scikit-learn.org/1.5/modules/generated/sklearn.cluster.KMeans.html}}
). We compare the methods based on their clustering performance in the second step using well-known metrics.

\textbf{Metrics.} Following prior work~\cite{jiang2020faster} we use Adjusted Rank Index (\textit{ARI}) and Adjusted Mutual Information (\textit{AMI}) to measure the quality of a clustering algorithms. We give the ground truth number of clusters to K-Means as input. 
% An ideal clustering should be able to recover the components exactly. 
ARI and AMI both vary between 0 and 1, with 0 corresponding to random clustering and 1 corresponding to perfect clustering. 

\textbf{Data generation.} For synthetic experiments we consider $K=10$ with equally likely component mixtures. For the CIFAR100 experiments, we randomly sample $K=20$ classes out of 100. See Suppl~\ref{sec:data_generation} for details. 

\textbf{Methods.} We consider learning a mapping matrix $\vA$ using gradient descent with respect to \underline{InfoNCE} and \underline{SimSiam} loss and we compare them againts 5 baselines. \underline{Ambient} is clustering in ambient dimension. \underline{Random} projects onto a random $r$-dimensional subspace. \underline{Optimal} projects onto top $r$ directions in Fisher subspace.
% (if $r>K$, we include $r-K$ directions orthogonal to $S_F$)
\underline{PCA} projects onto $r$-dimensional $S_{SVD}$. \underline{LDA} finds a projection that maximizes inter-class variance with respect to intra-class variance ({Optimal} in the synthetic case). 
% Note that InfoNCE and SimSiam also learn scaling within the subspace; baselines like PCA \emph{only} project.
% $(without orthonormalization, i.e. they also scale within subspace). 

\subsection{Synthetic experiments}
We observe that InfoNCE often performs better than Optimal, projection onto the Fisher subspace. We interpret that InfoNCE loss learns to \emph{scale} within the subspace leading to better clustering performance (see App~\ref{app:ortho_info} for orthonormalized plots of InfoNCE/SimSiam). Moreover, SimSiam shows a sub-par performance compared to InfoNCE.
% contrary to our theoretical results. 
We believe this is due the difficulty in optimizing the SimSiam objective (projected gradient descent) compared to InfoNCE. 

\textbf{Effect of noise in augmentations.} We vary the noise parameter $\delta$ for the AeD with everything else held constant. We can see in Fig~\ref{fig:delta_plot} that InfoNCE is robust to the variation in $\delta$. % and performs well even for small values of $\delta$.

\textbf{Effect of variance orthogonal to $S_F$.} Increasing the variance \emph{orthogonal to the Fisher subspace} is equivalent to making the pancakes in Fig~\ref{fig:fig1} flatter. We quantify \enquote{flatness} as $\sigma_{\mathrm{min}}(\vSigma)/\sigma_{\mathrm{max}}(\vSigma)$. Fig~\ref{fig:condition_plot} shows that InfoNCE is (almost) invariant to \enquote{flatness}, while PCA and Ambient fail when the variance orthogonal to Fisher subspace is large (i.e. $\sigma_{\mathrm{min}}(\Sigma)/\sigma_{\mathrm{max}}(\Sigma)$ is small). 

\textbf{Effect of projection dimension.} We vary the dimension of the lower dimensional space that we aim to learn. InfoNCE's performance increases with increasing $r$ and finally plateaus. 
% Interestingly, at $r=10$ there is a drop in the performance of Optimal. This is because we include a Optimal includes a random direction (since it has exhausted the fisher subspace in $9$ components). The inclusion of this random direction throws off the clustering algorithm. Hence, in the absence of knowledge of $r$, InfoNCE still gives a good subspace, while other methods (even those that include fisher subspace) can fail catastrophically. 

\textbf{Effect of variance within $S_F$.} We show that InfoNCE learns a "good" scaling within the Fisher subspace. We set the dimension of ambient space equal to the dimension of mean subspace (i.e. $d=K-1$). We select $K/2$ orthogonal directions and increase the variance in the subspace \emph{spanned by these direction} by a factor of $\sigma_{\mathrm{max}}(\Sigma)/\sigma_{\mathrm{min}}(\Sigma)$. Since none of the baselines learn scaling in the subspace, they perform almost the same, while InfoNCE outperforms all the baselines.

\subsection{Real-data experiments on CIFAR-100}
% We have conducted a new set of experiments using a real dataset; we have extended our evaluation to real datasets, i.e. instead of sampling from a synthetic GMM we consider images from CIFAR100 as input data. 
We consider images from CIFAR-100 as inputs instead of generating data from a synthetic GMM. We consider 20 fine-labeled classes (10K images) where each class corresponds to a single component. We convert images into a 256-dimensional vector by subsampling, grayscaling, mean scaling. Images belonging to the same class will be the augmentations of each other. We still consider a linear model and hence have the same baselines. 
We measure the clustering performance using ARI and AMI on mapping to $r=5,10,15,19$-dimensional subspaces in \cref{tab:cifar100}. Surprisingly, InfoNCE could still outperform LDA in certain metrics even in the real data setting where naturally, our data distribution assumptions and augmentation will no longer hold. LDA has the best performance for all dimensions in terms of ARI (first value in each column) which indicates that it excels at maintaining local groupings. On the other hand, InfoNCE achieves larger AMI values underlining that contrastive learning could be better at preserving overall class distributions. 

\begin{table}[t]
  \centering
  \caption{\label{tab:cifar100} \textbf{Clustering performance of linear dimensionality reduction methods on CIFAR100:} \textbf{Bolded} and \underline{underlined} values indicate the best and second-best scores for each column. We report clustering performance for K-Means on $r=5,10,15,19$-dimensional subspaces using the linear mappings learned by 6 different methods. We measure the clustering performance using ARI and AMI, which is reported as pairs, in the format \textbf{ARI | AMI}. InfoNCE and LDA outperform the rest by significant margins in all mapping dimensions. LDA shows the best performance in terms of ARI, while InfoNCE achieves higher AMI values across the board.}
  \label{tab:rme}
  \resizebox{0.95\textwidth}{!}{%
  \begin{tabular}{l@{\hspace{2em}}| c@{\hspace{2em}} c@{\hspace{2em}} c@{\hspace{2em}} c@{\hspace{2em}} c }
    \toprule
    & \multicolumn{5}{c}{\textbf{Mapping Dimension}} \\
  \cmidrule(lr){2-6}
    \multicolumn{1}{c|}{{\textbf{Method}}} & \textbf{5\,dim} & \textbf{10\,dim} & \textbf{15\,dim} & \textbf{19\,dim} & \textbf{30\,dim}\\
    \midrule
    Random   & 0.01475 | 0.04694 & 0.01567 | 0.05241 & 0.01894 | 0.05796 & 0.02446 | 0.07188 & 0.02503 | 0.07856\\
    PCA      & 0.03360 | 0.09254 & 0.03292 | 0.09509 & 0.03435 | 0.09427 & 0.03154 | 0.09269 & 0.03361 | 0.09562\\
    Ambient  & 0.03160 | 0.09481 & 0.03160 | 0.09481 & 0.03160 | 0.09481 & 0.03160 | 0.09481 & 0.03160 | 0.09481\\
    SimSiam  & 0.02870 | 0.09581 & 0.02732 | 0.09114 & 0.03707 | 0.10830 & 0.02868 | 0.09498 & 0.03562 | 0.10540\\
    InfoNCE  & \underline{0.05065} | \textbf{0.12730} & \underline{0.05138} | \textbf{0.12801} & \underline{0.05402} | \textbf{0.12835} & \underline{0.05285} | \textbf{0.12947} & \underline{0.05182} | \textbf{0.12548}\\
    LDA      & \textbf{0.05067} | \underline{0.12360} & \textbf{0.05489} | \underline{0.11782} & \textbf{0.06352} | \underline{0.12359} & \textbf{0.05970} | \underline{0.11933} & \textbf{0.05970} | \underline{0.11933}\\
    \bottomrule
  \end{tabular}%
}
\end{table}

\subsection{Real-data experiments on ImageNet}
We consider images from ImageNet as inputs instead of generating data from a synthetic GMM. We consider 20 fine-labeled classes (10K images) where each class corresponds to a single component, and images belonging to the same class serve as the augmentations. Since image classification is highly non-linear, we construct a setup to map them to a linearly seperable space. For this we represent each image using ResNet-50's last layer representations (i..e, 2048 dimensional vectors).
We measure the clustering performance using ARI and AMI on mapping to $r=5,10,15,19$-dimensional subspaces in \cref{tab:imagenet}. We can see that InfoNCE consistently outperforms LDA in this setting. 

\begin{table}[t]
  \centering
  \caption{\label{tab:imagenet} \textbf{Clustering performance of linear dimensionality reduction methods on ImageNet:} \textbf{Bolded} and \underline{underlined} values indicate the best and second-best scores for each column. We report clustering performance for K-Means on $r=5,10,15,19$-dimensional subspaces using the linear mappings learned by 6 different methods. We measure the clustering performance using ARI and AMI, which is reported as pairs, in the format \textbf{ARI | AMI}. InfoNCE and LDA outperform the rest by significant margins in all mapping dimensions. LDA shows the best performance in terms of ARI, while InfoNCE achieves higher AMI values across the board.}
  \label{tab:rme}
  \resizebox{0.95\textwidth}{!}{%
  \begin{tabular}{l@{\hspace{2em}}| c@{\hspace{2em}} c@{\hspace{2em}} c@{\hspace{2em}} c@{\hspace{2em}} c }
    \toprule
    & \multicolumn{5}{c}{\textbf{Mapping Dimension}} \\
  \cmidrule(lr){2-6}
    \multicolumn{1}{c|}{{\textbf{Method}}} & \textbf{5\,dim} & \textbf{10\,dim} & \textbf{15\,dim} & \textbf{19\,dim} & \textbf{30\,dim}\\
    \midrule
    Random   & 0.03026 | 0.08296 & 0.07231 | 0.14005 & 0.13140 | 0.23209 & 0.14405 | 0.23363 & 0.26177 | 0.38685\\
    PCA      & 0.23408 | 0.46447 & 0.39843 | 0.59117 & 0.49903 | 0.65898 & 0.50954 | 0.66159 & 0.56355 | 0.70163\\
    Ambient  & \underline{0.48641} | 0.67233 & 0.48641 | 0.67233 & 0.48641 | 0.67233 & 0.48641 | 0.67233 & 0.48641 | 0.67234\\
    SimSiam  & 0.37581 | 0.58233 & 0.56084 | 0.70246 & 0.58705 | 0.72227 & 0.60259 | 0.73159 & 0.62970 | 0.74476\\
    InfoNCE  & \textbf{0.84451} | \textbf{0.88831} & \textbf{0.98182} | \textbf{0.98115} & \textbf{0.92963} | \textbf{0.97390} & \textbf{0.99621} | \textbf{0.99579} & \underline{0.93317} | \textbf{0.97721}\\
    LDA      & 0.47112 | \underline{0.69828} & \underline{0.66967} | \underline{0.83895} & \underline{0.82550} | \underline{0.89103} & \underline{0.94744} | \underline{0.94990} & \textbf{0.94745} | \underline{0.94990}\\
    \bottomrule
  \end{tabular}%
}
\end{table}

\section{Conclusion and limitations}
We study contrastive learning in the classical setting of linear dimensionality reduction for GMMs for which we define a generalized notion of imperfect augmentations. Our main result underlines that the contrastive methods learn the Fisher-optimal subspace for the class of shared-covariance GMMs; going beyond the capabilities of unsupervised methods and matching that of fully supervised strategies. Our work particularly focuses on linear mappings, and it is an important open problem to theoretically verify performance of contrastive methods when projectors are non-linear, which is usually the case in practice. Similarly, we acknowledge that our results are not immediately generalizable for data distributions beyond GMMs, which we will investigate as future work.
% For future work, we hope to make use of our novel setting to design new, provable \emph{algorithms} for learning subspaces.

\bibliographystyle{plainnat}
\bibliography{neurips_2025}

%%%%%%%%%%%%%%%%%%%%%%%%%%%%%%%%%%%%%%%%%%%%%%%%%%%%%%%%%%%%

\newpage
\appendix
\onecolumn

\begin{center}
    {\LARGE \bf Appendix}
\end{center}

\section{Additional Lemmas}

\begin{lemma}\label{lemma:logsumexpconvex}
    The function given by $f(\{\vx_i\};\vc) = \log(\sum_i \exp(\vc^T\vx_i))$ where $\{\vx_i\},\vc \in \mathbb{R}^{d}$ is strictly convex in $\vc \in \mathrm{Span}(\{\vx_i\})$.
\end{lemma}

\begin{proof}
    Using the definition of strict convexity we need to prove that: 
    \begin{align*}
        & f(\{\vx_i\};(\lambda \vc_1 + (1-\lambda)\vc_2)) \\
        < & \lambda f(\{\vx_i\}; \vc_1) + (1-\lambda) f(\{\vx_i\};\vc_2) \\
    \end{align*}
    when $\vc_1 \neq \vc_2$. Taking $\exp(.)$ on both sides : 
    \begin{align*}
        & \exp(f(\{\vx_i\};(\lambda \vc_1 + (1-\lambda)\vc_2))) \\
        < & \exp(\lambda f(\{\vx_i\}; \vc_1)) \exp((1-\lambda)f(\{\vx_i\};\vc_2)) \\
    \end{align*}
    
    Simplifying the LHS we have : 
    \begin{align*}
        & \exp(f(\{\vx_i\};(\lambda \vc_1 + (1-\lambda)\vc_2))) \\
        =& \sum_i \exp((\lambda \vc_1 + (1-\lambda)\vc_2)^T\vx_i) \\
        =& \sum_i \exp(\lambda \vc_1^T\vx_i)\exp((1-\lambda)\vc_2^T\vx_i) \\
    \end{align*}
    
    Using Holders inequality ($\sum_i x_iy_i \leq (\sum_i |x_i|^p)^{\frac{1}{p}}(\sum_i |y_i|^q)^{\frac{1}{q}}$, where $\frac{1}{p}+\frac{1}{q} = 1$) with $p = \frac{1}{\lambda}$ and $q = \frac{1}{1-\lambda}$, we have : 
    
    \begin{align*}
    & \sum_i \exp(\lambda \vc_1^T\vx_i)\exp((1-\lambda)\vc_2^T\vx_i) \\
    \leq & (\sum_i \exp(\frac{\lambda \vc_1^T\vx_i}{\lambda}))^\lambda(\sum_i \exp(\frac{(1-\lambda)\vc_2^T\vx_i}{1-\lambda}))^{(1-\lambda)} \\
    = & (\sum_i \exp(\vc_1^T\vx_i))^\lambda(\sum_i \exp(\vc_2^T\vx_i))^{(1-\lambda)} \\
    % = & \exp(\lambda\log(\sum_i \exp(\vc_1^T\vx_i))) \exp((1-\lambda)\log(\sum_i \exp(\vc_2^T\vx_i))) \\
    =& \exp(\lambda f(\{\vx_i\}; \vc_1)) \exp((1-\lambda)f(\{\vx_i\};\vc_2)) \\
    \end{align*}

    Note that Holders equality holds only when $\vc_1^T\vx_i = \vc_2^T\vx_i$ for all $i$. For $\vc_1,\vc_2 \in \mathrm{Span}\{\vx_i\}$, this equality holds only when $\vc_1 = \vc_2$. Hence, for $\vc_1\neq \vc_2$ we have strict inequality i.e. 
    \begin{align*}
    & \sum_i \exp(\lambda \vc_1^T\vx_i)\exp((1-\lambda)\vc_2^T\vx_i)\\
    < & (\sum_i \exp(\frac{\lambda \vc_1^T\vx_i}{\lambda}))^\lambda(\sum_i \exp(\frac{(1-\lambda)\vc_2^T\vx_i}{1-\lambda}))^{(1-\lambda)} \\
     =& \exp(\lambda f(\{\vx_i\}; \vc_1)) \exp((1-\lambda)f(\{\vx_i\};\vc_2)) \\
    \end{align*}
    This gives strict convexity. Hence proved.
\end{proof}

\section{Main Proposition}\label{app:proofofprop}

In this section, we state and prove a key proposition which is used to prove our main result (Thm~\ref{thm:sharedcov_infonce})

\begin{proposition} \label{prop:isotropic_infonce}
Suppose $F$ parameterized by $\{w_k,\vmu_k,I\}_{k\in[K]}$ be a spherical gaussian mixture model and $\hat{F}$ be its augmentation-enabled gaussian mixture with bias $\delta$ (Def~\ref{def:jointdistribution}).
Let $S_F$ be the fisher subspace (Eqn~\ref{eqn:sfexplicit}) of $F$ and $\vA^*$ be the optimal solution of the InfoNCE loss (Eqn~\ref{eqn:infonce}) :
\begin{align*}
    \vA^* = \underset{\vA \in \mathbb{R}^{d\times r}}{\mathrm{argmin}}\  \mathcal{L}(\vA) 
\end{align*}
% Then given $r\geq K$, $\col{\vA^*} \subseteq \Span{\{\vmu_k\}_{k\in[K]}}$. Moreover if $\delta = 1$, then $\col{\vA^*} = \Span{\{\vmu_k\}_{k\in[K]}}$.
Then given $r\geq K$, $\col{\vA^*} \subseteq S_F$. Moreover if $\delta = 1$, then $\col{\vA^*} = S_F$.
\end{proposition}

We prove the proposition in two parts. In the first part we prove that $\col{\vA^*} \subseteq S_F$ for any $\delta>0$. In the second part we prove that if $\delta = 1$, then $\col{\vA^*} = S_F$.

\subsection{Column space of A is a subset of Fischer Subspace}

We prove that : $\col{\vA^*} \subseteq S_F$ when $\delta>0$
\begin{proof}

    \begin{align*}
        (\vA^T\vx)^T(\vA^T\vy) = \vx^T\vA\vA^T\vy = \vx^T\vB\vy,
    \end{align*}
    where $\vB=\vA\vA^T$, i.e. $\vB$ is positive semi-definite (PSD) matrix of rank $r$. 
    We substitute into InfoNCE loss to get : 
    \begin{align*}
    \mathcal{L} = -\mathbb{E}_{\vx,\aug{\vx}}[\vx^T\vB\aug{\vx}] + \mathbb{E}_{\vx} [\log(\mathbb{E}_{\negative{\vx}}[\exp(\vx^T\vB\negative{\vx})])] 
    % \mathcal{L} = & -\mathbb{E}_{\vx,\aug{\vx}}[f(\vx;\vA)^Tf(\aug{\vx};\vA)] \\ & + \mathbb{E}_{\vx} [\log(\mathbb{E}_{\negative{\vx}}[\exp(f(\vx;\vA)^T f(\negative{\vx};\vA))])] \\
    % = & -\mathbb{E}_{\vx,\aug{\vx}}[(\vA^T\vx)^T(\vA^T\aug{\vx})] \\ & + \mathbb{E}_{\vx} [\log(\mathbb{E}_{\negative{\vx}}[\exp((\vA^T\vx)^T(\vA^T\negative{\vx}))])] \\
    % = & -\mathbb{E}_{\vx,\aug{\vx}}[\vx^T\vA\vA^T\aug{\vx}] + \mathbb{E}_{\vx} [\log(\mathbb{E}_{\negative{\vx}}[\exp(\vx^T\vA\vA^T\negative{\vx})])] \\
    % = & -\mathbb{E}_{\vx,\aug{\vx}}[\vx^T\vB\aug{\vx}] + \mathbb{E}_{\vx} [\log(\mathbb{E}_{\negative{\vx}}[\exp(\vx^T\vB\negative{\vx})])] \\
    \end{align*}
    
    We relax the rank-constraint on $\vB$ throughout the proof. We show that rank of our optimal solution $\vB^* \leq K$ which satisfies the rank constraint implicitly (as $K\leq r$).

    \textbf{Note} : The above loss function is strictly convex in $\vB$ (using Lemma~\ref{lemma:logsumexpconvex}) and the minimization of is over a convex set (i.e. set of $\vB \in \mathbb{S}_{+}^d$).

    Let $\vB^*$ be the optimal solution. Denote the eigendecomposition of $\vB^*$ as $\vU\vLambda\vU^T$, where $\vLambda \succeq 0$ (as $\vB \succeq 0$) and $\vU$ is a unitary matrix. Equivalently : 
    \begin{align} \label{eigendecomp}
        \vB^* = \sum_i \lambda_i \vu_i \vu_i^T
    \end{align}
    where $\vu_i$ are columns of $\vU$. Consider the indices where the eigenvalue $\lambda_i >0$ as $\mathcal{I}$. The solution $\vA^*$ is hence $\left[ \sqrt{\lambda_i}\vu_i\right]_{i\in\mathcal{I}}$. We aim to show that $\col{\vA^*} = \mathrm{Span}\{\vu_i\}_{i \in \mathcal{I}} \subseteq \mathrm{Span}\{\vmu_k\}_{k \in [K]}$. 

    % The condition that columns of $\vB^*$ lie in $\mathrm{Span}\{\vmu_k\}$ implicitly implies that rank of $\vB^* \leq K$. This follows because there can't be more than $K$ orthogonal vectors (i.e. columns of $U$) in a subspace of dimension $K$ (as there are only $K$ means spanning the subspace). 
    The condition $\mathrm{Span}\{\vu_i\}_{i \in \mathcal{I}} \subseteq \mathrm{Span}\{\vmu_k\}_{k \in [K]}$ implicitly implies that rank of $\vB^* \leq K$. This follows because there can't be more than $K$ orthogonal vectors (i.e. columns of $U$) in a subspace of dimension $K$ (as there are only $K$ means spanning the subspace).

    Suppose the condition is not true. Then there exists a unit vector $\vv$, such that $\vv^T\vmu_k = 0$ for all $k \in [K]$ and $\vv^T\vu_i \neq 0$ for some $i$ where $\lambda_i>0$.  

    We construct a new matrix $\bar{\vU}$, whose columns are reflection of $\vU$ through the plane with normal vector $\vv$ given by $\vR = \vI-2\vv\vv^T$. The reflection matrix is defined such that $\vR\vmu_k = \vR^T\vmu_k = \vmu_k$. Define $\bar{\vB}$ from the constructed $\bar{\vU}$.
    \begin{align*}
        \bar{\vU} &= \vR\vU = (\vI-2\vv\vv^T)\vU \\
        \bar{\vB} &= \bar{\vU}\vLambda \bar{\vU}^T = \vR\vB^*\vR^T\\
    \end{align*}
    $\bar{\vB} \neq \vB$. $\bar{\vU} \neq \vU $ is still a unitary matrix (product of unitary matrices), is identical to $\vU$ in $\mathrm{Span}\{\vmu_k\}_{k \in [K]}$
    \begin{align*}
        % & \vU'^T\vU' = \vU^T(\vI-2\vv\vv^T)(\vI-2\vv\vv^T)\vU \\
        % = & \vU^T(\vI + 4\vv\vv^T - 2\vv\vv^T - 2\vv\vv^T)\vU = \vU^T\vU = \vI \\
        \bar{\vU}^T\bar{\vU} & = \vU^T\vR^T\vR\vU = \vU^T\vU = \vI \\
    % \vU'^T\vmu_k &= \vU^T(I-2\vv\vv^T)\vmu_k \\
    % & = \vU^T(\vmu_k-2\vv\vv^T\vmu_k) = \vU^T\vmu_k\\
    \bar{\vU}^T\vmu_k &= \vU^T\vR^T\vmu_k  = \vU^T\vmu_k\\
    \bar{\vU}\vmu_k &= \vU\vR\vmu_k  = \vU\vmu_k\\
    \end{align*}

    The first term in the loss $\mathcal{L}$ at $\vB^*$ can be simplified as
    \begin{align*}
        & \mathbb{E}_{\vx,\aug{\vx}}[\vx^T\vB^*\aug{\vx}]\\
        =& \mathbb{E}_{\vx,\aug{\vx}}[\langle \vx\aug{\vx}^T,\vB^* \rangle] \\
        =& \langle \mathbb{E}_{\vx,\aug{\vx}}[\vx\aug{\vx}^T],\vB^* \rangle \\
        % =& \langle -\sum_k \vmu_k \vmu_k^T,\vB^* \rangle \\
        =& \langle \sum_k w_k\vmu_k (\delta\vmu_k + (1-\delta)(\sum_j w_j \vmu_j))^T,\vB^* \rangle \\
        =& \langle \delta \sum_k w_k\vmu_k\vmu_k^T + (1-\delta)(\sum_k w_k\vmu_k)(\sum_j w_j \vmu_j)^T,\vB^* \rangle \\
    \end{align*}

    This is where we use the fact that the augmentation is generated from the augmentation oracle. 

    We can now see that 
    % \begin{align*}
    %     & \mathbb{E}_{\vx,\aug{\vx}}[\vx^T\vB'\aug{\vx}] \\
    %     = & \langle -\sum_k \vmu_k \vmu_k^T,\vB' \rangle \\
    %     = & \langle -\sum_k \vmu_k \vmu_k^T,(\vI-2\vv\vv^T)\vB^*(\vI-2\vv\vv^T) \rangle \\
    %     = & \langle -\sum_k (\vI-2\vv\vv^T)\vmu_k \vmu_k^T(\vI-2\vv\vv^T),\vB^* \rangle \\
    %     = & \langle -\sum_k \vmu_k \vmu_k^T,\vB^* \rangle \\
    %     = & \mathbb{E}_{\vx,\aug{\vx}}[\vx^T\vB^*\aug{\vx}] \\
    % \end{align*}

    \begin{align*}
        & \mathbb{E}_{\vx,\aug{\vx}}[\vx^T\bar{\vB}\aug{\vx}] \\
        =& \langle \delta \sum_k w_k\vmu_k\vmu_k^T + (1-\delta)(\sum_k w_k\vmu_k)(\sum_j w_j \vmu_j)^T,\bar{\vB} \rangle \\
        =& \langle \delta \sum_k w_k\vmu_k\vmu_k^T + (1-\delta)(\sum_k w_k\vmu_k)(\sum_j w_j \vmu_j)^T,\vR\vB\vR^T \rangle \\
        =& \langle \delta \sum_k w_k\vR^T\vmu_k\vmu_k^T\vR + \\
         & (1-\delta)(\sum_k w_k\vR^T\vmu_k)(\sum_j w_j \vR^T\vmu_j)^T,\vB^* \rangle \\
        =& \langle \delta \sum_k w_k\vmu_k\vmu_k^T + (1-\delta)(\sum_k w_k\vmu_k)(\sum_j w_j \vmu_j)^T,\vB^* \rangle \\
        =& \mathbb{E}_{\vx,\aug{\vx}}[\vx^T\vB^*\aug{\vx}]
    \end{align*}

    Now see analyze the second term in $\mathcal{L}$ where we prove : 

    \begin{align*}
        \mathbb{E}_{\vx} [\log(\mathbb{E}_{\negative{\vx}}[\exp(\vx^T\bar{\vB}\negative{\vx})])] & = \mathbb{E}_{\vx} [\log(\mathbb{E}_{\negative{\vx}}[\exp(\vx^T\vB^*\negative{\vx})])] \\
    \end{align*}

    For this we show that the random variables $\vx^T\vB^*\negative{\vx}$ and $\vx^T\bar{\vB}\negative{\vx}$ have identical distribution. We first simplify $\vx^T\vB^*\negative{\vx}$ below.

    % \begin{align*}
    %     & \vx^T\vB'\negative{\vx} \\ 
    %     = & \vx^T(\vI-2\vv\vv^T) \vB^* (\vI-2\vv\vv^T) \negative{\vx} \\ 
    %     = & (\vx-\vmu_z + \vmu_z)^T(\vI-2\vv\vv^T) \vB^*(\vI-2\vv\vv^T)(\negative{\vx}-\vmu_{\negative{z}} + \vmu_{\negative{z}})\\ 
    %     = & ((\vx-\vmu_z)^T(\vI-2\vv\vv^T) + \vmu_z^T(\vI-2\vv\vv^T)) \vB^*((\vI-2\vv\vv^T)(\negative{\vx}-\vmu_{\negative{z}}) + (\vI-2\vv\vv^T)\vmu_{\negative{z}}) \\ 
    %     = & ((\vx-\vmu_z)^T(\vI-2\vv\vv^T) + \vmu_z^T) \vB^*((\vI-2\vv\vv^T)(\negative{\vx}-\vmu_{\negative{z}}) +\vmu_{\negative{z}}) \\
    %     = & ((\vx-\vmu_z)^T\vR + \vmu_z^T) \vB^*(\vR(\negative{\vx}-\vmu_{\negative{z}}) +\vmu_{\negative{z}})
    % \end{align*}

    \begin{align*}
        & \vx^T\bar{\vB}\negative{\vx} \\ 
        = & \vx^T \vR \vB^* \vR \negative{\vx} \\ 
        = & (\vx-\vmu_z + \vmu_z)^T \vR \vB^* \vR(\negative{\vx}-\vmu_{\negative{z}} + \vmu_{\negative{z}})\\ 
        = & ((\vx-\vmu_z)^T \vR + \vmu_z^T \vR) \vB^*( \vR(\negative{\vx}-\vmu_{\negative{z}}) +  \vR\vmu_{\negative{z}}) \\ 
        = & ((\vx-\vmu_z)^T \vR + \vmu_z^T) \vB^*( \vR(\negative{\vx}-\vmu_{\negative{z}}) +\vmu_{\negative{z}}) \\
    \end{align*}

    where $\vmu_z^T\vR = \vmu_z^T(\vI-2\vv\vv^T) = \vmu_z^T-2(\vmu_z^T\vv)\vv^T = \vmu_z^T$. Now we prove that their distributions are identical.

    \begin{align*}
        & \mathrm{Pr}(\vx^T\bar{\vB}\negative{\vx} \leq c) \\ 
        = & \mathrm{Pr}(((\vx-\vmu_z)^T\vR + \vmu_z^T) \vB^*(\vR(\negative{\vx}-\vmu_{\negative{z}}) +\vmu_{\negative{z}}) \leq c) \\ 
        = & \mathrm{Pr}(((\vx-\vmu_z)^T + \vmu_z^T) \vB^*((\negative{\vx}-\vmu_{\negative{z}}) +\vmu_{\negative{z}}) \leq c) \\ 
        = & \mathrm{Pr}(\vx\vB^*\negative{\vx} \leq c) 
    \end{align*}

    We use the fact that $(\vx-\vmu_z)^T(\vI-2\vv\vv^T)$ is identically distributed to $(\vx-\vmu_z)$ i.e. $\mathcal{N}(0,\vI)$. 

    Now since $\vx^T\vB^*\negative{\vx},\vx^T\bar{\vB}\negative{\vx}$ are identically distributed, second term in the loss are also identically distributed and have the same expectation i.e.
    \begin{align*}
        \mathbb{E}_{\vx} [\log(\mathbb{E}_{\negative{\vx}}[\exp(\vx^T\vB^*\negative{\vx})])] = \mathbb{E}_{\vx} [\log(\mathbb{E}_{\negative{\vx}}[\exp(\vx^T\bar{\vB}\negative{\vx})])]
    \end{align*} 

    This proves that $\mathcal{L}$ is identical for $\vB^*$ and $\vB'$. But since our loss is strictly convex we have 
    $\mathcal{L}(\frac{\vB^*+\bar{\vB}}{2}) < \frac{1}{2}(\mathcal{L}(\vB^*) + \mathcal{L}(\bar{\vB})) = \mathcal{L}(\vB^*)$. 
    % \pari{$\vB$ belongs to the span of $\vx\negative{\vx}^T?$}
    This contradicts the fact that $\vB^*$ is optimal in $\mathbb{S}^d_{+}$.
    Hence proved.
    \end{proof}

\subsection{Column space of A is equal to Fischer Subspace}

We prove that : $\col{\vA^*} = S_F$ when $\delta=1$

\begin{proof}
Consider the eigendecomposition for the optimal solution $\vB^*$ (Eq~\ref{eigendecomp}). Suppose there is a direction $\vv \in \Span{\{\vmu_k\}_k}$ which is not in $\Span{\{\vu_i\}_{i \in \mathcal{I}}}$. 

Without loss of generality assume $\vv^T\vu_i = 0 \ \forall i \in \mathcal{I}$ (if not then project onto the null space and re-normalize). Now since $\lambda_j = 0 \  \forall i \notin \mathcal{I}$, we can rotate the eigenvectors $[\vu_j]_{j \notin \mathcal{I}}$ with a unitary matrix such that $\vu_j= v$ for some $j$. This operation doesn't change $\vB^*$

That implies there exists $\vu_j$ such that $\vu_j \in \Span{\{\vmu_k\}_k}$ and $\lambda_j = 0$. We now show that 
\begin{align*}
    \frac{\partial \mathcal{L}}{\partial \lambda_j}\Bigr\rvert_{\vB = \vB^*} < 0 
\end{align*}
Hence $\lambda_j >0$ for optimal $\vB^*$.

First consider the derivative of first term of $\mathcal{L}$ with $\lambda_j$
\begin{align*}
% & \frac{\partial \mathbb{E}_{\vx,\aug{\vx}}[\vx^T\vB\aug{\vx}]}{\partial \lambda_j} \\
% =& \langle \delta \sum_k w_k\vmu_k\vmu_k^T + (1-\delta)(\sum_k w_k\vmu_k)(\sum_j w_j \vmu_j)^T,\frac{\partial \vB}{\partial \lambda_j} \rangle \\
% =& \langle \delta \sum_k w_k\vmu_k\vmu_k^T,\vu_j\vu_j^T \rangle +\\
% &\langle (1-\delta)(\sum_k w_k\vmu_k)(\sum_j w_j \vmu_j)^T ,\vu_j\vu_j^T \rangle \\
% =& \delta \sum_k w_k(\vmu_k^T\vu_j)^2 +(1-\delta)  (\sum_k w_k\vmu_k^T\vu_j)^2\\
& \frac{\partial \mathbb{E}_{\vx,\aug{\vx}}[\vx^T\vB\aug{\vx}]}{\partial \lambda_j} = \langle \sum_k w_k\vmu_k\vmu_k^T,\frac{\partial \vB}{\partial \lambda_j} \rangle \\
=& \langle \sum_k w_k\vmu_k\vmu_k^T,\vu_j\vu_j^T \rangle = \sum_k w_k(\vmu_k^T\vu_j)^2 = \sum_k w_k a_k^2
\end{align*}

where $a_k = \vmu_k^T\vu_j$. Since $\vu_j \in \Span{\{\vmu_k\}_k}$ that implies there exists a non-zero $a_k$. Hence we have that $\frac{\partial \mathbb{E}_{\vx,\aug{\vx}}[\vx^T\vB\aug{\vx}]}{\partial \lambda_j}>0$ for any $\vB\neq 0$. 

For the second term in the loss we have : 
\begin{align*}
    & \frac{\partial \mathbb{E}_{\vx} [\log(\mathbb{E}_{\negative{\vx}}[\exp(\vx^T\vB\negative{\vx})])]}{\partial \lambda_j} \\
    = & \mathbb{E}_{\vx} \bigg[\frac{\mathbb{E}_{\negative{\vx}}[\exp(\vx^T\vB\negative{\vx})\frac{\partial \vx^T\vB\negative{\vx}}{\partial \lambda_j}]}{\mathbb{E}_{\negative{\vx}}[\exp(\vx^T\vB\negative{\vx})]}\bigg] \\
    = & \mathbb{E}_{\vx} \bigg[\frac{\mathbb{E}_{\negative{\vx}}[\exp(\vx^T\vB\negative{\vx})(\vx^T\vu_j)(\negative{\vx}^T\vu_j)]}{\mathbb{E}_{\negative{\vx}}[\exp(\vx^T\vB\negative{\vx})]}\bigg] \\
    = & \mathbb{E}_{\vx} \bigg[(\vx^T\vu_j)\frac{\mathbb{E}_{\negative{\vx}}[\exp(\vx^T\vB\negative{\vx})(\negative{\vx}^T\vu_j)]}{\mathbb{E}_{\negative{\vx}}[\exp(\vx^T\vB\negative{\vx})]}\bigg] \\
    = & \mathbb{E}_{\vx} \Bigg[(\vx^T\vu_j)\mathbb{E}_{\negative{\vx}}\bigg[\frac{\exp(\vx^T\vB\negative{\vx})}{\mathbb{E}_{\negative{\vx}}[\exp(\vx^T\vB\negative{\vx})]}\negative{\vx}^T\vu_j\bigg]\Bigg] \\
    = & \mathbb{E}_{\vx} \Bigg[(\vx^T\vu_j)\mathbb{E}_{\negative{\vx}}\bigg[g(\negative{\vx};\vx,\vB)\negative{\vx}^T\vu_j\bigg]\Bigg] 
\end{align*}
where we define $g(\negative{\vx};\vx,\vB)$ as :
\begin{align*}
g(\negative{\vx};\vx,\vB) \triangleq \frac{\exp(\vx^T\vB\negative{\vx})}{\mathbb{E}_{\negative{\vx}}[\exp(\vx^T\vB\negative{\vx})]}
\end{align*}

with $\mathbb{E}_{\negative{\vx}}[g(\negative{\vx};\vx,\vB)] = 1$. Also define $\vx_\perp = \vu_j(\vu_j^T\vx)$ and $\vx_\parallel = (I-\vu_j\vu_j^T)\vx$. Hence $\vx = \vx_\perp + \vx_\parallel$. Notice that since $\lambda_j = 0$ for $\vB^*$, we have 
\begin{align*}
    g(\negative{\vx};\vx,\vB^*) = g(\negative{\vx}_\parallel;\vx,\vB^*) = g(\negative{\vx}_\parallel;\vx_\parallel,\vB^*)
\end{align*}

We evaluate the inner term in the above expression at $\vB = \vB^*$ : 
\begin{align*}
& \mathbb{E}_{\negative{\vx}}\bigg[g(\negative{\vx};\vx,\vB^*)\negative{\vx}^T\vu_j \bigg] \\
=  & \sum_k w_k \mathbb{E}_{\negative{\vx}\sim\mathcal{N}(\vmu_k,I)}\bigg[g(\negative{\vx};\vx,\vB^*)\negative{\vx}^T\vu_j \bigg] \\
= & \sum_k w_k \mathbb{E}_{\negative{\vx}\sim\mathcal{N}(\vmu_k,I)}\bigg[g(\negative{\vx};\vx,\vB^*)\negative{\vx}^T\vu_j \bigg]
\end{align*}
We look at each of the expectations individually 

\begin{align*}
    & \mathbb{E}_{\negative{\vx}\sim\mathcal{N}(\vmu_k,I)}\bigg[g(\negative{\vx};\vx,\vB^*)\negative{\vx}^T\vu_j \bigg] \\
    = & \mathbb{E}_{\negative{\vx}\sim\mathcal{N}(\vmu_k,I)}\bigg[g(\negative{\vx};\vx,\vB^*)(a_k+\negative{\vz}^T\vu_j) \bigg] \\ 
    = & \mathbb{E}_{\negative{\vx}\sim\mathcal{N}(\vmu_k,I)}\bigg[g(\negative{\vx};\vx,\vB^*)a_k \bigg]+ \\
    & \mathbb{E}_{\negative{\vx}\sim\mathcal{N}(\vmu_k,I)}\bigg[g(\negative{\vx};\vx,\vB^*)(\negative{\vz}_\perp^T\vu_j) \bigg] \\ 
    = & a_k \mathbb{E}_{\negative{\vx}\sim\mathcal{N}(\vmu_k,I)}\bigg[g(\negative{\vx};\vx,\vB^*) \bigg] = a_k h_k(\vx;\vB^*) \\
\end{align*}
where $\negative{\vz}$ is the noise in $\negative{\vx}$ and $h_k(\vx;\vB^*) = \mathbb{E}_{\negative{\vx}\sim\mathcal{N}(\vmu_k,I)}\bigg[g(\negative{\vx};\vx,\vB^*) \bigg]$.  We use the property that $\negative{\vz}_\perp$ is independent  of $g(\negative{\vx};\vx,\vB^*) = g(\negative{\vx}_\parallel;\vx,\vB^*)$ to show that its equal to 0.
\begin{align*}
    & \mathbb{E}_{\negative{\vx}\sim\mathcal{N}(\vmu_k,I)}\bigg[g(\negative{\vx}_\parallel;\vx,\vB^*)(\negative{\vz}_\perp^T\vu_j) \bigg] \\
    = & \mathbb{E}_{\negative{\vx}\sim\mathcal{N}(\vmu_k,I)}\bigg[g(\negative{\vx}_\parallel;\vx,\vB^*)\bigg]\mathbb{E}_{\negative{\vx}\sim\mathcal{N}(\vmu_k,I)}\bigg[\negative{\vz}_\perp^T\vu_j \bigg] \\
    = & \mathbb{E}_{\negative{\vx}\sim\mathcal{N}(\vmu_k,I)}\bigg[g(\negative{\vx}_\parallel;\vx,\vB^*)\bigg]*0 = 0 
\end{align*}

Hence we have 
\begin{align*}
& \mathbb{E}_{\negative{\vx}}\bigg[g(\negative{\vx};\vx,\vB^*)\negative{\vx}^T\vu_j \bigg] = \sum_k w_k a_k h_k(\vx;\vB^*)
\end{align*}

We have $\sum_k w_k h_k(\vx) = 1$
\begin{align*}
    & \sum_k w_k h_k(\vx;\vB)  \\
    = & \sum_k w_k \mathbb{E}_{\negative{\vx}\sim\mathcal{N}(\vmu_k,I)}\bigg[g(\negative{\vx}_\parallel;\vx,\vB) \bigg] \\
    = & \sum_k w_k \mathbb{E}_{\negative{\vx}\sim\mathcal{N}(\vmu_k,I)}\bigg[\frac{\exp(\vx^T\vB\negative{\vx})}{\mathbb{E}_{\negative{\vx}}[\exp(\vx^T\vB\negative{\vx})]} \bigg] \\
    = & \mathbb{E}_{\negative{\vx}}\bigg[\frac{\exp(\vx^T\vB\negative{\vx})}{\mathbb{E}_{\negative{\vx}}[\exp(\vx^T\vB\negative{\vx})]} \bigg] = 1 \\
\end{align*}
While $g(\negative{\vx};\vx,\vB)$, is the weight $\vx$ gives to a point $\negative{\vx}$, $h_k(\vx;\vB)$ can be interpreted as a weight $\vx$ gives to cluster $k$  (i.e. expectation of $g$ over a cluster).

Going back to the expression at $\vB = \vB^*$
\begin{align*}
& \mathbb{E}_{\vx} \Bigg[(\vx^T\vu_j)\mathbb{E}_{\negative{\vx}}\bigg[g(\negative{\vx};\vx,\vB^*)\negative{\vx}^T\vu_j\bigg]\Bigg] \\
= & \mathbb{E}_{\vx} \Bigg[(\vx^T\vu_j)\sum_k w_k a_k h_k(\vx;\vB^*)\Bigg] \\
= & \sum_k w_k a_k \mathbb{E}_{\vx} \bigg[(\vx^T\vu_j) h_k(\vx;\vB^*)\bigg] \\
= & \sum_k w_k a_k \mathbb{E}_{\vx} \bigg[(\vx^T\vu_j) h_k(\vx;\vB^*)\bigg] \\
= & \sum_k w_k a_k \sum_{k'} w_{k'} \mathbb{E}_{\vx\sim\mathcal{N}(\vmu_{k'},I)}\bigg[(\vx^T\vu_j) h_k(\vx;\vB^*)\bigg]
\end{align*}

$h$ inherits the property $h_k(\vx;\vB^*) = h_k(\vx_\parallel;\vB^*)$ from $g$ (this is because $\lambda_j = 0$ in $\vB^*$).
Evaluating the inner expression we have 
\begin{align*}
    & \mathbb{E}_{\vx\sim\mathcal{N}(\vmu_{k'},I)}\bigg[(\vx^T\vu_j) h_k(\vx;\vB^*)\bigg] \\
    = & \mathbb{E}_{\vx\sim\mathcal{N}(\vmu_{k'},I)}\bigg[(a_{k'}+\vz_\perp^T\vu_j) h_k(\vx;\vB^*)\bigg] \\
    = & a_{k'}\mathbb{E}_{\vx\sim\mathcal{N}(\vmu_{k'},I)}\bigg[h_k(\vx;\vB^*)\bigg] \\
    + & \mathbb{E}_{\vx\sim\mathcal{N}(\vmu_{k'},I)}\bigg[(\vz_\perp^T\vu_j) h_k(\vx;\vB^*)\bigg] \\
    = & a_{k'}\mathbb{E}_{\vx\sim\mathcal{N}(\vmu_{k'},I)}\bigg[h_k(\vx;\vB^*)\bigg] = a_{k'} f_{k',k} \\
\end{align*}

where $\vz$ is the noise in $\vx$. We use the property that $\vz_\perp$ is independent  of $h_k(\vx;\vB^*) = h_k(\vx_\parallel;\vB^*)$ to show that its equal to 0. 
\begin{align*}
    \sum_{k} w_{k} f_{k',k} = & \sum_{k} w_{k} \mathbb{E}_{\vx\sim\mathcal{N}(\vmu_{k'},I)}\bigg[h_{k}(\vx;\vB^*)\bigg] \\
     & \mathbb{E}_{\vx\sim\mathcal{N}(\vmu_{k'},I)}\bigg[\sum_{k} w_{k} h_{k}(\vx;\vB^*)\bigg] = 1 \\
\end{align*}
since $\sum_k w_k h_k(\vx;\vB^*) = 1$.

Define a matrix $\vF\in \mathbb{R}^{K\times K}$ as $\vF_{i,j} = f_{i,j}$ and $\vW = \mathrm{diag}(w_1,w_2,\ldots,w_{K})$ and $\va = [a_1,a_2,\ldots,a_{K}]^T$.

\begin{align*}
    & \mathbb{E}_{\vx} \Bigg[(\vx^T\vu_j)\mathbb{E}_{\negative{\vx}}\bigg[g(\negative{\vx};\vx,\vB^*)\negative{\vx}^T\vu_j\bigg]\Bigg] \\
    = & \sum_k w_k a_k \sum_{k'} w_{k'} \mathbb{E}_{\vx\sim\mathcal{N}(\vmu_{k'},I)}\bigg[(\vx^T\vu_j) h_k(\vx;\vB^*)\bigg] \\
    = & \sum_k w_k a_k \sum_{k'} w_{k'} a_{k'} f_{k',k} \\
    = & \va^T\vW \vF \vW \va = (\sqrt{\vW}\va)^T \sqrt{\vW}\vF \sqrt{\vW} (\sqrt{\vW}\va)
\end{align*}

Eigenvalues of $\sqrt{\vW}\vF \sqrt{\vW}$ are equal to eigenvalues of $\vF\vW$ (since this is a similarity transform). And if $\vv$ is eigenvector of $\vF\vW$, then $\sqrt{\vW}\vv$ is the eigenvector for $\sqrt{\vW}\vF \sqrt{\vW}$. 

Since $\vF$ has expectation of exponentials as its entries i.e. $\vF_{i,j}>0$ and $\vW>0$ from definition. Hence we have that every entry of $\vF\vW$ is strictly greater than 0. The Perron–Frobenius eigenvalue $r$ is given by 
\begin{align*}
    \underset{i}{\mathrm{min}} \sum_j (\vF\vW)_{i,j} \leq r \leq \underset{i}{\mathrm{max}} \sum_j (\vF\vW)_{i,j}
\end{align*}

But we have that for every $i$ the sum is equal to $1$. 
\begin{align*}
    \sum_j (\vF\vW)_{i,j} = \sum_j \sum_l (\vF_{i,l}\vW_{l,j})  = \sum_j f_{i,j}w_j = 1
\end{align*}

Hence $r = 1$ and the Perron vector is simply $\mathbf{1} = [1,1,\ldots,1] \in \mathbb{R}^K$. Hence we show that eigenvector with largest eigenvalue (i.e. 1) for $\sqrt{\vW}\vF \sqrt{\vW}$ is $\sqrt{\vW}\mathbf{1}$. Hence we have that 

\begin{align*}
    (\sqrt{\vW}\va)^T \sqrt{\vW}\vF \sqrt{\vW} (\sqrt{\vW}\va) \leq  ||\sqrt{\vW}\va||^2 = \sum_k w_k a_k^2
\end{align*}
The equality only holds when $\va \in \Span{\mathbf{1}}$, i.e. all $a_k$ are equal. This is not true (since $\sum_k a_k = 0$ and there exists a nonzero $a_k$). Hence 
\begin{align*}
    \mathbb{E}_{\vx} \Bigg[(\vx^T\vu_j)\mathbb{E}_{\negative{\vx}}\bigg[g(\negative{\vx};\vx,\vB^*)\negative{\vx}^T\vu_j\bigg]\Bigg] < \sum_k w_k a_k^2
\end{align*}

The loss for the whole term is \begin{align*}
\frac{\partial \mathcal{L}}{\partial \lambda_j}\Bigr\rvert_{\vB = \vB^*} = & - \frac{\partial \mathbb{E}_{\vx,\aug{\vx}}[\vx^T\vB\aug{\vx}]}{\partial \lambda_j}\\
& + \frac{\partial \mathbb{E}_{\vx} [\log(\mathbb{E}_{\negative{\vx}}[\exp(\vx^T\vB\negative{\vx})])]}{\partial \lambda_j} \\
 < & -\sum_k w_k a_k^2 + \sum_k w_k a_k^2 = 0 
\end{align*}
Hence proved. 

\end{proof}

\section{Proof for InfoNCE loss}\label{app:proofofmain}
In this section, we use \cref{prop:isotropic_infonce} to prove our main theorem. We recommend the reader to first go through \cref{prop:isotropic_infonce} and it's proof (in \cref{app:proofofprop}) to understand the proof for the main theorem.

\begin{proof}

\begin{align*}
\mathcal{L} = -\mathbb{E}_{\vx,\aug{\vx}}[\vx^T\vB\aug{\vx}] + \mathbb{E}_{\vx} [\log(\mathbb{E}_{\negative{\vx}}[\exp(\vx^T\vB\negative{\vx})])] \\
\end{align*}
where $\vB=\vA\vA^T$, i.e. $\vB$ is positive semi-definite (PSD) matrix of rank $r$. The above step follows from Proposition~\ref{prop:isotropic_infonce}. We relax the rank-constraint on $\vB$. We finally show that rank of $\vB^*$ i.e. the optimal solution $\leq K$ which satisfies the above condition as $K\leq r$.

Now consider a change of variables given by $\vx' = \vSigma^{-\half}\vx$. This implies that $\vx'$ now follows a GMM with means given by $\{\vSigma^{-\half}\vmu_k\}_{k \in [K]}$ and isotropic covariance $\vI$ (as $\mathbb{E}[\vx'\vx'^T] = \mathbb{E}[\vSigma^{-\half}\vx\vx^T\vSigma^{-\half}] = \vSigma^{-\half}\mathbb{E}[\vx\vx^T]\vSigma^{-\half}= \vI$).

The loss be hence be written as : 

\begin{align*}
\mathcal{L} = & -\mathbb{E}_{\vx',\aug{\vx}'}[\vx'^T\vSigma^{\half}\vB\vSigma^{\half}\aug{\vx}'] \\
& + \mathbb{E}_{\vx'} [\log(\mathbb{E}_{\negative{\vx}'}[\exp(\vx'^T\vSigma^{\half}\vB\vSigma^{\half}\negative{\vx}')])] \\
\end{align*}

Now let $\vB' = \vSigma^{\half}\vB\vSigma^{\half}$. $\vB'$ is also a PSD matrix. Hence loss can be written as : 

\begin{align*}
\mathcal{L} = -\mathbb{E}_{\vx',\aug{\vx}'}[\vx'^T\vB'\aug{\vx}'] + \mathbb{E}_{\vx'} [\log(\mathbb{E}_{\negative{\vx}'}[\exp(\vx'^T\vB'\negative{\vx}')])] \\
\end{align*}

Let the optimal $\vB'^*$ be denoted by $\sum_i \lambda_i \vu_i \vu_i^T$ and let $\mathcal{I}$ be set of indices with $\lambda_i>0$ (note $\lambda_i \geq 0 \forall i$). But from Proposition~\ref{prop:isotropic_infonce}, we know that , where $\mathrm{Span}\{\vu_i\}_{i \in \mathcal{I}} = \mathrm{Span}\{\vmu'_k\}_{k \in [K]}$ where $\{\vmu'_k\}$ is the set of means.  

For an invertible $\vSigma$ and $\vmu'_k = \vSigma^{-\half}\vmu_k$, we have

\begin{align*}
    &\mathrm{Span}\{\vu_i\}_{i \in \mathcal{I}} = \mathrm{Span}\{\vmu'_k\} \\
    \implies & \mathrm{Span}\{\vu_i\}_{i \in \mathcal{I}} = \mathrm{Span}\{\vSigma^{-\half}\vmu_k\} \\
    \implies & \mathrm{Span}\{\vSigma^{-\half}\vu_i\}_{i \in \mathcal{I}} = \mathrm{Span}\{\vSigma^{-1}\vmu_k\} 
\end{align*}

Substituting and $\vB^* = \vSigma^{-\half}\vB'^*\vSigma^{-\half}$ we get : 
\begin{align*}
    \vB^*  & = \vSigma^{-\half}(\sum_i \lambda_i \vu_i \vu_i^T)\vSigma^{-\half} \\
     & = \sum_i (\sqrt{\lambda_i}\vSigma^{-\half}\vu_i) (\sqrt{\lambda_i}\vSigma^{-\half}\vu_i)^T \\
\end{align*}

% The optimal matrix in original space $\vB^*$ is given as : $\vSigma^{-\half}\vB'^*\vSigma^{-\half} = \vSigma^{-\half}\vU\vLambda\vU^T\vSigma^{-\half} = (\vSigma^{-\half}\vU\sqrt{\vLambda})(\vSigma^{-\half}\vU\sqrt{\vLambda})^T = \vA^*\vA^{*T}$. 

Hence the optimal solution in the original space $\vA^* =\left[ \sqrt{\lambda_i}\vSigma^{-\half}\vu_i\right]_{i\in\mathcal{I}}$. We proved that $\mathrm{Span}\{\vSigma^{-\half}\vu_i\}_{i \in \mathcal{I}} = \mathrm{Span}\{\vSigma^{-1}\vmu_k\}_{k \in [K]}$. This implies that column space of $\col{\vA^*} = \mathrm{Span}\{\vSigma^{-\half}\vu_i\}_{i \in \mathcal{I}} = \mathrm{Span}\{\vSigma^{-1}\vmu_k\}_{k \in [K]}$. Hence proved.

% Since the columns of $U$ are in the subspace $\{\vSigma^{-\half}\vmu_k\}_{k \in [K]}$, the columns of $\vA^*$ are in $\{\vSigma^{-\half}\vSigma^{-\half}\vmu_k\}_{k \in [K]} = \{\vSigma^{-1}\vmu_k\}$. 

\end{proof}

\section{Proof for SimSiam Loss}\label{app:simsiamproof}
In this section, we use the same methodology as in part of \cref{prop:isotropic_infonce} to prove the theorem. First we use the strict convexity of the loss to show that the solution lies in the fisher subspace. Afterwards, by differentiating the loss function w.r.t. any direction in the fisher subspace, we show that all directions in the subspace should be learnt assuming sufficient capacity. 

\begin{proof}
    We write the objective by computing the expectations as : 
    \begin{align*}
        \mathcal{L}_{SS}(\vA)  & = \underset{(\vx,\aug{\vx}) \sim \hat{F}}{-\ \mathbb{E}}\Bigg[(\vA^T\vx)^T\vA^T\aug{\vx}\Bigg] + \xi \underset{\vx\sim F}{\mathbb{E}} \Bigg[||\vA^T\vx||^2\Bigg] \\
         & = \underset{(\vx,\aug{\vx}) \sim \hat{F}}{-\ \mathbb{E}}\Bigg[ \langle \vA\vA^T,\aug{\vx}\vx^T \rangle\Bigg] + \xi \underset{\vx\sim F}{\mathbb{E}} \Bigg[\langle \vA\vA^T,\vx\vx^T \rangle\Bigg] \\
         & = -\langle \vA\vA^T,\underset{(\vx,\aug{\vx}) \sim \hat{F}}{\mathbb{E}} [ \aug{\vx}\vx^T ] \rangle + \xi \langle \vA\vA^T,\underset{\vx\sim F}{\mathbb{E}} [\vx\vx^T ] \rangle \\
         & = -\langle \vA\vA^T,\delta \underset{k}{\sum} w_k\vmu_k\vmu_k^T \rangle + \xi \langle \vA\vA^T, \underset{k}{\sum} w_k\vmu_k\vmu_k^T + \vSigma \rangle \\
         & = \langle \vB,(-\delta + \xi) \vM \rangle + \xi \langle \vB, \vSigma \rangle \\
         & = \langle \vB,(-\delta + \xi) \vM + \xi \vSigma \rangle\\
         & = \langle \vSigma^{\frac{1}{2}}\vB\vSigma^{\frac{1}{2}},(-\delta + \xi)\vSigma^{-\frac{1}{2}}\vM\vSigma^{-\frac{1}{2}} + \xi \vI \rangle \\
         & = \langle \vB',(-\delta + \xi)\vM' + \xi \vI \rangle \\
    \end{align*}
    where we have $\vB = \vA\vA^T$, $\vM = \underset{k}{\sum} w_k\vmu_k\vmu_k^T$, $\vB' = \vSigma^{\frac{1}{2}}\vB\vSigma^{\frac{1}{2}}$ and $\vM' = \vSigma^{-\frac{1}{2}}\vM\vSigma^{-\frac{1}{2}}$. Note that the loss is linear in $\vB$ (and $\vB'$). Since we restrict ourselves to $||\vA||_2\leq 1$, it is also true that $||\vB||_2 \leq 1$. 

    We now show that the optimal $\vB'$ has column space (and row space, since symmetric) lying in column space (and row space) of $\vM'$. 
    Let $\vB^*$ be the optimal solution. Denote the eigendecomposition of $\vB^*$ as $\vU\vLambda\vU^T$, where $\vLambda \succeq 0$ (as $\vB \succeq 0$) and $\vU$ is a unitary matrix. Equivalently : 
    \begin{align} \label{eigendecomp}
        \vB^* = \sum_i \lambda_i \vu_i \vu_i^T
    \end{align}
    where $\vu_i$ are columns of $\vU$.
    % Consider the indices where the eigenvalue $\lambda_i >0$ as $\mathcal{I}$. The solution $\vA^*$ is hence $\left[ \sqrt{\lambda_i}\vu_i\right]_{i\in\mathcal{I}}$. We aim to show that $\col{\vA^*} = \mathrm{Span}\{\vu_i\}_{i \in \mathcal{I}} \subseteq \mathrm{Span}\{\vmu_k\}_{k \in [K]}$. 
    % The condition $\mathrm{Span}\{\vu_i\}_{i \in \mathcal{I}} \subseteq \mathrm{Span}\{\vmu_k\}_{k \in [K]}$ implicitly implies that rank of $\vB^* \leq K$. This follows because there can't be more than $K$ orthogonal vectors (i.e. columns of $U$) in a subspace of dimension $K$ (as there are only $K$ means spanning the subspace). 
    Suppose the condition is not true. Then there exists a unit vector $\vv$, such that $\vM'\vv = 0$  and $\vB'\vv \neq 0$. 

    We construct a new matrix $\bar{\vU}$, whose columns are projection of $\vU$ onto the plane with normal vector $\vv$ given by $\vR = \vI-\vv\vv^T$. The projection matrix is defined such that $\vR\vM' = \vM'$. Define $\bar{\vB}$ from the constructed $\bar{\vU}$.
    \begin{align*}
        \bar{\vU} &= \vR\vU = (\vI-\vv\vv^T)\vU \\
        \bar{\vB} &= \bar{\vU}\vLambda \bar{\vU}^T = \vR\vB^*\vR^T\\
    \end{align*}

    Now through some algebra we see that loss at $\bar{\vB}$ is less than at $\vB^*$ and hence $\vB^*$ can't be optimal. 
    
    \begin{align*}
        & \langle \bar{\vB},(-\delta + \xi)\vM' + \xi \vI \rangle \\
        = & \langle \vB^*,(-\delta + \xi)\vR^T\vM'\vR + \xi \vR^T\vR \rangle \\
        = & \langle \vB^*,(-\delta + \xi)\vM' + \xi \vI + \xi (\vR^T\vR-\vI) \rangle \\
        = & \langle \vB^*,(-\delta + \xi)\vM' + \xi \vI \rangle + \xi \langle \vB^*,(\vR^T\vR-\vI) \rangle \\
    \end{align*}

    Now we show that $\langle \vB^*,(\vR^T\vR-\vI) \rangle  < 0$ and since $\xi >0$, loss at $\bar{\vB}$ is less than at $\vB^*$. The fact is intuitively clear and aa formal proof is as follows : 

    \begin{align*}
        \langle \vB^*,\vR^T\vR \rangle =  \tr(\vR\vB^*\vR^T) = \sum_i \lambda_i ||\vR\vu_i||_2^2 < \sum_i \lambda_i ||\vu_i||_2^2 = \langle \vB^*,\vI \rangle
    \end{align*}

    The strict inequality is due to fact that $\vB'\vv \neq 0$, i.e. there exists a $i$ with $\lambda_i>0$ and $\vv^T\vu_i \neq 0$ (and hence $||\vR\vu_i||_2^2 < ||\vu_i||_2^2$). $||\vR\vu_j||_2^2 \leq ||\vu_j||_2^2$ is true generally because $\vR$ is a projection matrix. 

    Hence we showed that column space of $\vB^*$ (i.e., optimal $\vB'$) lies in column space of $\vM'$. Now we show that it \textbf{spans the whole column space}. Suppose not. Let $\vB^*$ be the optimal solution with decomposition with notation as used above. Then WLOG there exists $\vu_i$ which $\in \Span{\vM'}$ and $\lambda_i = 0$ (if it doesn't exist we can rotate the $\vu_j$'s with 0 eigenvalues so that a $\vu_i$ aligns in the subspace). We take derivative w.r.t. $\lambda_i$ and show that it's negative. 

    \begin{align*}
        \frac{\partial \mathcal{L}_{SS}}{\partial \lambda_i}\Bigr\rvert_{\vB' = \vB^*} = & \frac{\partial \langle \vB',(-\delta + \xi)\vM' + \xi \vI \rangle}{\partial \lambda_i}\\
         = &  \vu_i^T\big((-\delta + \xi)\vM' + \xi \vI\big)\vu_i\\
         = &  (-\delta + \xi)\vu_i^T\vM'\vu_i + \xi\\
         % \leq &  (-\delta + \xi) \underset{k}{\sum} w_k (\vmu_k^T\vSigma^{-\frac{1}{2}}\vu_i)^2 + \xi\\
         % \leq &  (-\delta + \xi) \underset{k}{\sum} w_k ||\vSigma^{-\frac{1}{2}}\vmu_k||_2^2 + \xi\\
    \end{align*}
    % if $\xi < \frac{\delta\underset{k}{\sum} w_k ||\vSigma^{-\frac{1}{2}}\vmu_k||_2^2}{1+\underset{k}{\sum} w_k ||\vSigma^{-\frac{1}{2}}\vmu_k||_2^2}$ then $\frac{\partial \mathcal{L}_{SS}}{\partial \lambda_i}$ is 0 
    if $\xi < \frac{\delta\vv^T\vM'\vv}{1+\vv^T\vM'\vv}$ for all directions $\vv$ in $\Span{\vM'}$ (as $\vu_i \in \Span{\vM'}$), then $\frac{\partial \mathcal{L}_{SS}}{\partial \lambda_i}$ is $< 0$. 
    $\frac{\delta\vv^T\vM'\vv}{1+\vv^T\vM'\vv}$ is monotonic in $\vv^T\vM'\vv$. The minimum value of  $\vv^T\vM'\vv$ is the smallest non-zero eigenvalue of $\vM'$ denoted by $\lambda_{\mathrm{min}}$. Hence if $0<\xi<\frac{\delta\lambda_{\mathrm{min}}}{1+\lambda_{\mathrm{min}}}$ we are good. 

    Now we showed that $\vB^*$ (i.e., optimal $\vB'$) spans the complete column space of $\vM'$. Hence using the facts $\vB' = \vSigma^{\frac{1}{2}}\vB\vSigma^{\frac{1}{2}}$ and $\vB = \vA^T\vA$, we can argue that for the optimal value of $\vA$ denoted by $\vA^*$, $\col{\vA^*} = S_F$.
    \end{proof}
\section{Proof for CLIP Loss}\label{app:clipproof}
In this section, we use the first part of \cref{prop:isotropic_infonce} to prove the theorem, i.e. we use the proof of $\col{\vA^*} \subseteq S_F$. 

\begin{proof}
    We write the objective by substituting the functional form of $f$ as : 

    \begin{align*}
        \mathcal{L} = - & \mathbb{E}_{\vx_t,\vx_v}[\vx_{t}^T\vA_{t}\vA^T_v\vx_v] \\
        + & \mathbb{E}_{\vx_t} [\log(\mathbb{E}_{\negative{\vx}_v}[\exp(\vx_t^T\vA_{t}\vA^T_v\negative{\vx}_v)])] \\
        = - & \mathbb{E}_{\vx_t,\vx_v}[\vx^T_t\vB\vx_v] + \mathbb{E}_{\vx_t} [\log(\mathbb{E}_{\negative{\vx}_v}[\exp(\vx^T_t\vB\negative{\vx}_v)])] \\
    \end{align*}
    where $\vB = \vA_t\vA^T_v$. We can further do a change of variable by $\vx_t' = \vSigma_{t}^{-\half}\vx_t$ and $\vx_v' = \vSigma_{v}^{-\half}\vx_v$. The means are now $\vmu_{t,k}' = \vSigma_{t}^{-\half}\vmu_{t,k}$ and $\vmu_{v,k}' = \vSigma_{v}^{-\half}\vmu_{v,k}$. Define $\vB' = \vSigma_{t}^{\half}\vB\vSigma_{v}^{\half}$. Now $\vx_t'$ and $\vx_v'$ have components with covariance being isotropic. Now we have : 
    \begin{align*}
        \mathcal{L} & = -\mathbb{E}_{\vx_t,\vx_v}[\vx^T_t\vB\vx_v] + \mathbb{E}_{\vx_t} [\log(\mathbb{E}_{\negative{\vx}_v}[\exp(\vx^T_t\vB\negative{\vx}_v)])] \\
        & = -\mathbb{E}_{\vx_t',\vx_v'}[\vx'^T_t\vB'\vx_v'] + \mathbb{E}_{\vx_t'} [\log(\mathbb{E}_{\negative{\vx}_v'}[\exp(\vx'^T_t\vB'\negative{\vx}_v')])]
    \end{align*}

    We argue that optimal $\vB'^*$ has its row space equal to  $\mathrm{Span}\{\vmu_{v,k}'\}$ and its column space equal to $\mathrm{Span}\{\vmu_{t,k}'\}$.
    We present the argument for column space (row space argument follows similarly). 
    
    Suppose there exists a unit vector $\vv$ such that $\vv^T\vmu_{t,k}' = 0$ for all $k \in [K]$, $\vv^T\vB'^* \neq 0$. Then we can define a new matrix as $\bar{\vB}' = \vR\vB'^* = (I-2\vv\vv^T)\vB'^*$, where $\vR = I-2\vv\vv^T$ is a reflection matrix. 
    Following arguments from Proposition~\ref{prop:isotropic_infonce} we can prove that $\vx'^T_t\vB'^*\negative{\vx}_v'$ is identically distributed to $\vx'^T_t\bar{\vB}'\negative{\vx}_v'$. Hence using this we can show that $\mathcal{L}(\vB'^*)=\mathcal{L}(\bar{\vB}')$. 
    Then by the strict convexity of the loss function $\mathcal{L}$ we have that $\mathcal{L}(\frac{\vB'^*+\bar{\vB}'}{2}) < \frac{\mathcal{L}(\vB'^*)+\mathcal{L}(\bar{\vB}')}{2} = \mathcal{L}(\vB'^*)$. Hence $\vB'^*$ can't be optimal. 

    % Similarly suppose that $\lambda_{j} = 0$ for $\vu_{t,j} \in \mathrm{Span}\{\vmu_{t,k}'\}$ and $\vu_{v,j} \in \mathrm{Span}\{\vmu_{v,k}'\}$. Then again following steps from Proposition~\ref{prop:isotropic_infonce} we can show that 
    % \pari{need to check this}
    % \begin{align*}
    %     \frac{\partial \mathcal{L}}{\partial \lambda_j}\Bigr\rvert_{\vB = \vB^*} < 0 
    % \end{align*}
    Hence proved.
\end{proof}

\section{Representation collapse in non-contrastive learning}\label{sec:simsiam_eg}

Consider a two component GMM in a $d$ dimensional ambient space. The means of the components lie on the $x$ and $y$ axis (i.e. the first two dimensions), equidistant from the origin. Both components have isotropic covariance. We plot the first two dimensions below in~\cref{fig:fig_simsiam}. 
\begin{figure}[h]
\centering
\includegraphics[width=0.5\textwidth]{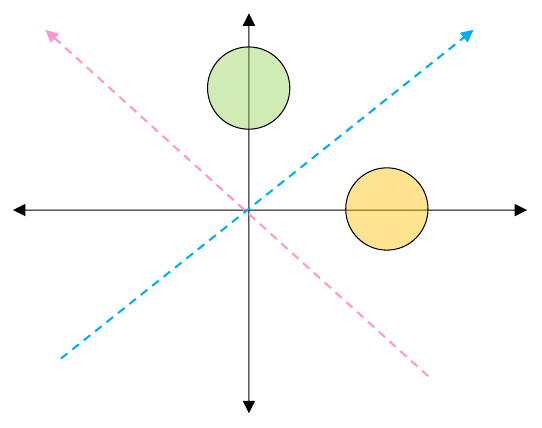}
\caption{The means of the GMM components in the first two dimensions.}
\label{fig:fig_simsiam}
\end{figure}
Note that for this setup the rank of the fisher subspace is two (i.e., $K=2$) and the fisher subspace is the $x-y$ plane. 

Consider learning a mapping matrix onto a one-dimensional subspace (i.e. $r<K$). InfoNCE type contrastive objectives would learn the subspace given by the pink line (i.e. gaussian would be projected onto the pink line and hence would be well seperated). This is stated without proof, but we know that InfoNCE loss would learn a subspace lying in the fisher subspace. And through some basic algebra we can convince ourselves that the solution would be the pink like (i.e. the line $x+y = 0$). 

But for non-contrastive objectives like SimSiam, while we can only prove that the optimal solution lies in the $x-y$ plane. SimSiam objective is not able to distinguish between the lines $x+y=0$ and $x=y$, and hence might lead to collapse of representations. But as stated in the \cref{thm:sharedcov_infonce}, this is only the case if $r<K$. For $r\geq K$, the SimSiam objective learns the complete Fisher subspace (and only the fisher subspace). 
\section{Experimental Details and Additional Results}
\begin{figure*}[t]
%%%%%%%%%%%
\begin{subfigure}{0.3\linewidth}
\includegraphics[width=\linewidth]{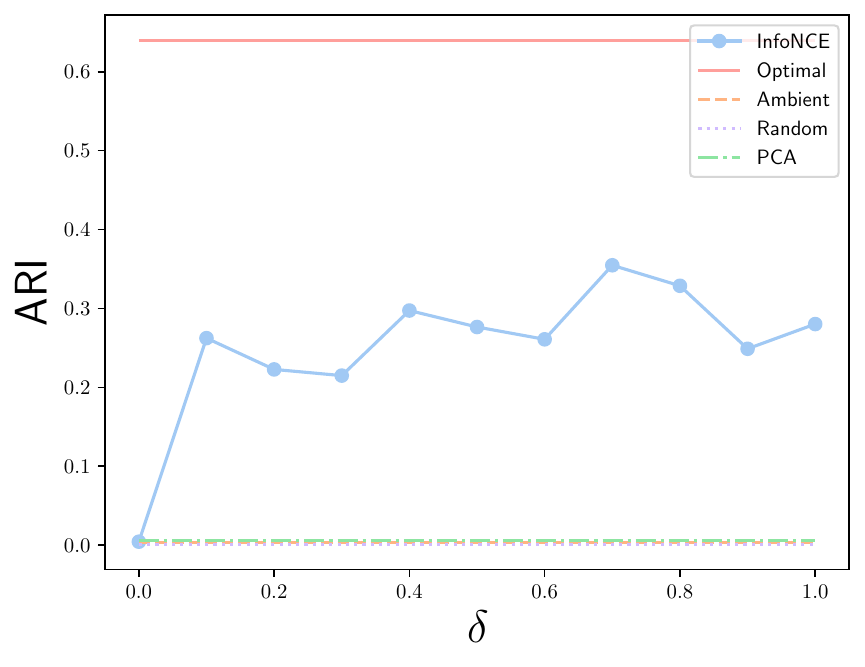}\\
\includegraphics[width=\linewidth]{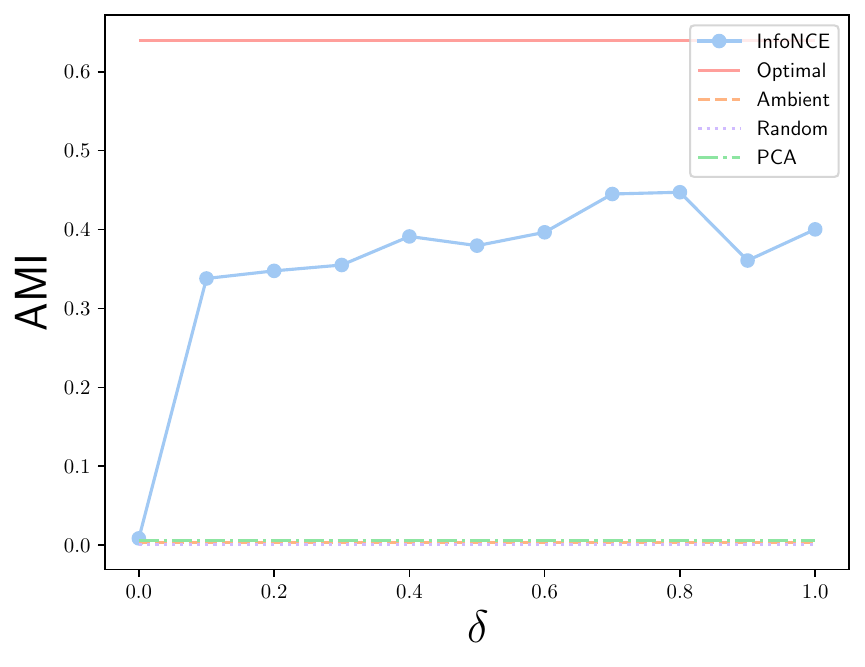}
\caption{Noise in augmentation}
\label{fig:delta_plot}
\end{subfigure}
%%%%%%%%%%%
\begin{subfigure}{0.3\linewidth}
\includegraphics[width=\linewidth]{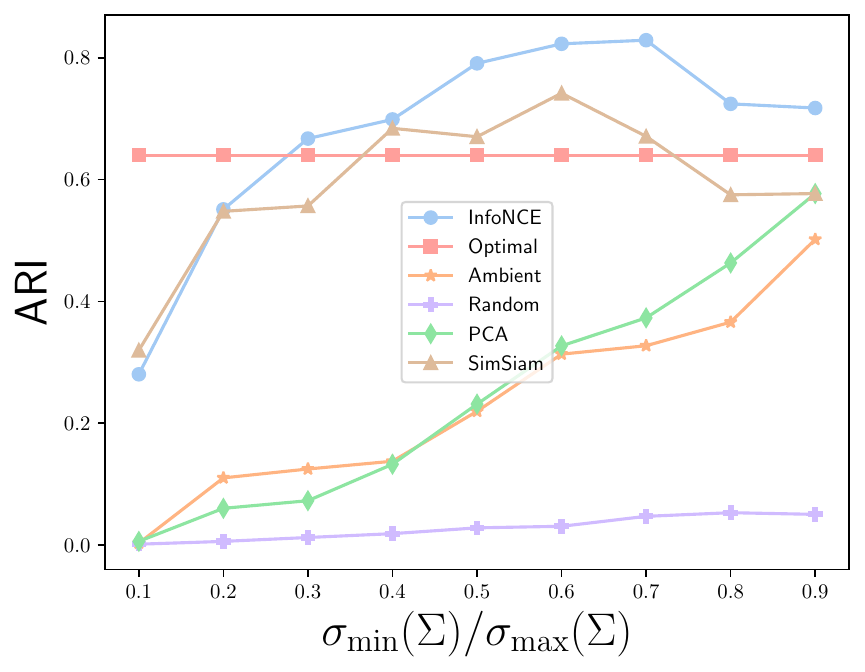}
\includegraphics[width=\linewidth]{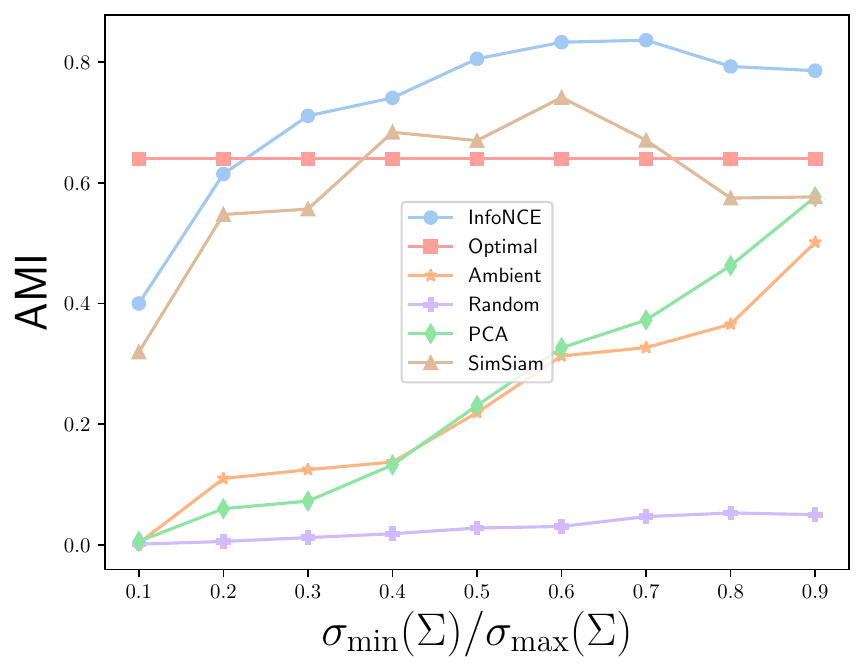}
\caption{Variance orthogonal to $S_F$}
\label{fig:condition_plot}
\end{subfigure}
%%%%%%%%%%%
\begin{subfigure}{0.3\linewidth}
\includegraphics[width=\linewidth]{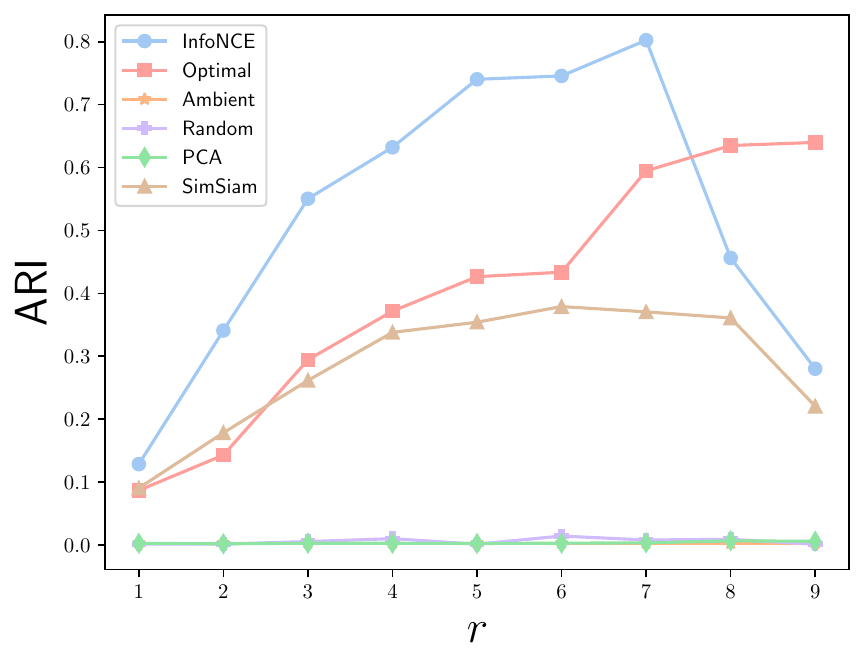}
\includegraphics[width=\linewidth]{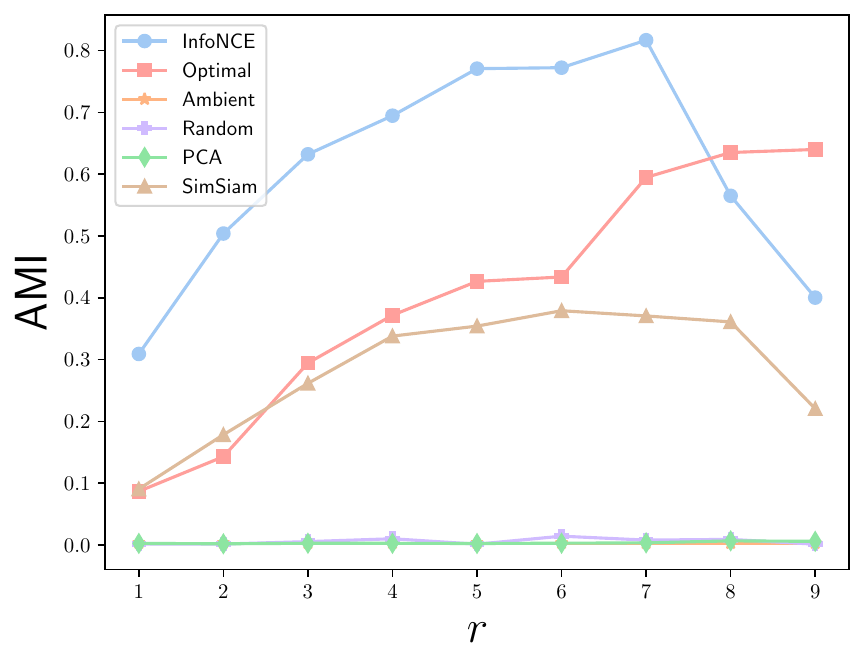}
\caption{Dimension of projected space}
\label{fig:rank_plot}
\end{subfigure}
%%%%%%%%%%%
%%%%%%%%%%%
\caption{\label{fig:mi_plots_ortho} We further extend the results presented in \cref{fig:mi_plots} with orthonormalized InfoNCE and SimSiam mappings}
\end{figure*}

\subsection{Data Generation for Synthetic Experiments} \label{sec:data_generation}

\textbf{Ambient Dimension} We consider the ambient dimension to be 100 for all experiments (except for scaling curves in \cref{fig:scaling_plot}, where we let ambient dimension equal to $K-1$). 

\textbf{Means of Components} For generating means of the components, we first sample $K-1$ times from a 0 mean, $\vI$ normal distribution in ambient dimension space. We chose the $K$th mean such that the means sum upto 0. Hence the means lie in a $K-1=9$ dimensional space.

\textbf{Covariance Matrix} To construct a covariance matrix we start with a unit variance $\vI$ matrix. We then upscale the variance in (Ambient - NumMeans+1) dimensional space orthogonal to the means subspace by factor of $\kappa$, where $\kappa = \sigma_{\mathrm{max}}(\vSigma)/\sigma_{\mathrm{min}}(\vSigma)$. We take $\kappa$ to be $1/0.1 = 10.0$ by default. For scaling experiments, we choose random NumMeans/2 orthogonal directions, and increase the covariance in those directions. Note for condition number experiment we take $\kappa$ to be $1/0.x$ where $x$ goes from 1 to 9. 

\subsection{Orthonormalized InfoNCE and SimSiam}\label{app:ortho_info}
We plot the ARI/AMI numbers for orthonormalized InfoNCE and SimSiam matrices in \cref{fig:mi_plots_ortho}. Specifically, we do a QR decomposition on these matrices and take the orthogonal Q matrix as the projection matrix. In these plots the convergence issues of InfoNCE loss become more apparent, as we can see that InfoNCE lagging behind optimal. 
% \paragraph{Effect of $\delta$} We can see that InfoNCE is worse than Optimal now (since there is not scaling), but still it is much better than other baselines. InfoNCE loss converged to a minimum (in population setting) should give the same ARI/AMI as Optimal. 
% \paragraph{Effect of $\sigma_{\mathrm{min}}(\vSigma)/\sigma_{\mathrm{max}}(\vSigma)$} We can see that InfoNCE is (almost) invariant to noise in orthogonal to the fisher subspace, while PCA and Ambient fail when the variance orthogonal to fisher subspace is large (i.e. $\sigma_{\mathrm{min}}(\Sigma)/\sigma_{\mathrm{max}}(\Sigma)$ is small). 
% \paragraph{Effect of $r$} We vary the dimension of the lower dimensional space we aim to learn. For larger values of $r$ ($r>7$), the directions are mostly noise and hence InfoNCE tends to assign lower weight to those directions. But orthonormalizing leads to these directions being equally represented, leading to decrease in AMI/ARI. Again if trained till convergence to minimum, this would not be observed. 
% Interestingly, at $r=10$ there is a drop in the performance of Optimal. This is because we include a Optimal includes a random direction (since it has exhausted the fisher subspace in $9$ components). The inclusion of this random direction throws off the clustering algorithm. Hence, in the absence of knowledge of $r$, InfoNCE still gives a good subspace, while other methods (even those that include fisher subspace) can fail catastrophically. 

\subsection{CIFAR-100 Experiments and Additional Results}
In this section, we present a seeded version of our results in the manuscript for the CIFAR-100 experiments. We randomly sample 20 classes from the datasets and collect the data points that belong to these 20 classes. The randomly selected subset of images are then normalized before inputting the the mapping method. Once the methods learn linear maps, it is then applied on the normalized data. Finally, we cluster the (linearly) transformed data using K-Means where the number of classes is the same as the number of components in the GMM. We repeat this procedure for 5 times with different seeds for the random number generator. We report the results in \cref{tab:cifar-appendix}, which are consistent with the results in the main text.

\begin{table}[t]
  \centering
  \caption{\textbf{Clustering performance of linear dimensionality reduction methods on CIFAR100:} \textbf{Bolded} and \underline{underlined} values indicate the best and second-best scores for each column. We report clustering performance for K-Means on $r=5,10,15,19$-dimensional subspaces (mapping dimensions) using the linear mappings learned by 6 different methods. We measure the clustering performance using ARI and AMI. We run the methods for 5 different seeds, corresponding to 5 random 20-class subsets of CIFAR-100 datasets. The average scores over 5 runs are reported with standard deviation given in paranthesis. InfoNCE and LDA outperform the rest by significant margins in all mapping dimensions. LDA shows the best performance in terms of ARI, while InfoNCE achieves higher AMI values across the board.}
  \label{tab:cifar-appendix}
  \resizebox{0.99\textwidth}{!}{%
\begin{tabular}{l|cc|cc|cc|cc}
\toprule
& \multicolumn{2}{c|}{\textbf{Mapping Dim = 5}} 
& \multicolumn{2}{c|}{\textbf{Mapping Dim = 10}} 
& \multicolumn{2}{c|}{\textbf{Mapping Dim = 15}} 
& \multicolumn{2}{c}{\textbf{Mapping Dim = 19}} \\
\cmidrule(lr){2-9}
\textbf{Method} & \textbf{ARI} & \textbf{AMI} 
& \textbf{ARI} & \textbf{AMI} 
& \textbf{ARI} & \textbf{AMI} 
& \textbf{ARI} & \textbf{AMI} \\
\midrule
Random  & 0.020 {\tiny (0.003)} & 0.061 {\tiny (0.009)} & 0.026 {\tiny (0.003)} & 0.078 {\tiny (0.009)} & 0.028 {\tiny (0.003)} & 0.084 {\tiny (0.007)} & 0.033 {\tiny (0.007)} & 0.094 {\tiny (0.015)} \\
Ambient & 0.039 {\tiny (0.006)} & 0.113 {\tiny (0.013)} & 0.039 {\tiny (0.006)} & 0.113 {\tiny (0.013)} & 0.039 {\tiny (0.006)} & 0.113 {\tiny (0.013)} & 0.039 {\tiny (0.006)} & 0.113 {\tiny (0.013)} \\
PCA     & 0.038 {\tiny (0.006)} & 0.110 {\tiny (0.013)} & 0.040 {\tiny (0.006)} & 0.113 {\tiny (0.014)} & 0.040 {\tiny (0.006)} & 0.114 {\tiny (0.014)} & 0.041 {\tiny (0.007)} & 0.116 {\tiny (0.015)} \\
SimSiam & 0.035 {\tiny (0.009)} & 0.111 {\tiny (0.020)} & 0.037 {\tiny (0.009)} & 0.113 {\tiny (0.021)} & 0.035 {\tiny (0.010)} & 0.110 {\tiny (0.021)} & 0.035 {\tiny (0.009)} & 0.110 {\tiny (0.018)} \\
InfoNCE & \underline{0.058} {\tiny (0.009)} & \textbf{0.146} {\tiny (0.016)} & \underline{0.064} {\tiny (0.011)} & \textbf{0.156} {\tiny (0.018)} & \underline{0.063} {\tiny (0.005)} & \textbf{0.157} {\tiny (0.014)} & \underline{0.066} {\tiny (0.008)} & \textbf{0.159} {\tiny (0.014)} \\
LDA     & \textbf{0.060} {\tiny (0.008)} & \underline{0.145} {\tiny (0.014)} & \textbf{0.074} {\tiny (0.012)} & \underline{0.153} {\tiny (0.013)} & \textbf{0.074} {\tiny (0.010)} & \underline{0.148} {\tiny (0.014)} & \textbf{0.076} {\tiny (0.014)} & \underline{0.149} {\tiny (0.017)} \\
\bottomrule
\end{tabular}
}
\end{table}

% \begin{tabular}{l@{\hspace{2em}}| c@{\hspace{2em}} c@{\hspace{2em}} c@{\hspace{2em}} c }
%     \toprule
%     & \multicolumn{4}{c}{\textbf{Mapping Dimension}} \\
%   \cmidrule(lr){2-5}
%     \multicolumn{1}{c|}{{\textbf{Method}}} & \textbf{5\,dim} & \textbf{10\,dim} & \textbf{15\,dim} & \textbf{19\,dim}\\
%     \midrule
%     Random   & 0.01475 | 0.04694 & 0.01567 | 0.05241 & 0.01894 | 0.05796 & 0.02446 | 0.07188 \\
%     PCA      & 0.03360 | 0.09254 & 0.03292 | 0.09509 & 0.03435 | 0.09427 & 0.03154 | 0.09269 \\
%     Ambient  & 0.03160 | 0.09481 & 0.03160 | 0.09481 & 0.03160 | 0.09481 & 0.03160 | 0.09481 \\
%     SimSiam  & 0.02870 | 0.09581 & 0.02732 | 0.09114 & 0.03707 | 0.10830 & 0.02868 | 0.09498 \\
%     InfoNCE  & \underline{0.05065} | \textbf{0.12730} & \underline{0.05138} | \textbf{0.12801} & \underline{0.05402} | \textbf{0.12835} & \underline{0.05285} | \textbf{0.12947} \\
%     LDA      & \textbf{0.05067} | \underline{0.12360} & \textbf{0.05489} | \underline{0.11782} & \textbf{0.06352} | \underline{0.12359} & \textbf{0.05970} | \underline{0.11933} \\
%     \bottomrule
%   \end{tabular}
\section{Linear Discriminant Analysis (LDA) and Fisher LDA} \label{app:lda}

\subsection{Linear Discriminant Analysis}

The classical binary class LDA objective is defined as a bayes-optimal solution for classification under the assumption that the data is generated from a two-component Gaussian mixture model with identical covariances. For non-shared covariance for two component case we don't have a closed form solution and the problem is referred to as Quadratic discriminant Analysis. The LDA solution for two-component GMM with shared covariance is the subspace projection on which leads to no loss in likelihood of data. It is given as : 

\begin{align*}
    \mathcal{S}_{LDA} = \vSigma^{-1}(\vmu_1-\vmu_2)
\end{align*}

\subsection{Fischer Linear Discriminant Analysis}

Instead of defining the LDA objective to be the Bayes-optimal under the assumption that data is generated from a two-component Gaussian mixture model with shared covariance, Fischer LDA considers an alternative objective. It does away with both the GMM and the shared covariance assumption and assumes data to be coming from two different distributions, where each distribution belonged is defined by it's class. Fisher defined the separation between these two distributions to be the ratio of the variance between the classes to the variance within the classes i.e. 
\begin{align*}
    \mathcal{S}_{fischer} = \frac{|\vtheta^T(\vmu_1-\vmu_2)|^2}{|\vtheta^T(w_1 \vSigma_1 + w_2 \vSigma_2)\vtheta|}
\end{align*}

The solution maximizing this is given by $(\vSigma_1 + \vSigma_2)^{-1}(\vmu_1-\vmu_2)$ is termed the Fisher subspace. When the data is generated from a two-component shared covariance GMM, the solution coincides with $\vSigma^{-1}(\vmu_1-\vmu_2)$ learnt by LDA.

\newpage
\section*{NeurIPS Paper Checklist}

\begin{enumerate}

\item {\bf Claims}
    \item[] Question: Do the main claims made in the abstract and introduction accurately reflect the paper's contributions and scope?
    \item[] Answer: \answerYes{} % Replace by \answerYes{}, \answerNo{}, or \answerNA{}.
    \item[] Justification: We propose a standard theory framework in machine learning, and provide theoretical results that validate our claims. We also conduct experiments to corroborate out theoretical claims.

\item {\bf Limitations}
    \item[] Question: Does the paper discuss the limitations of the work performed by the authors?
    \item[] Answer: \answerYes{} % Replace by \answerYes{}, \answerNo{}, or \answerNA{}.
    \item[] Justification: We have a section for limitations at the end of the paper. We also follow up the theorem statements with discussions on the implications and shortcomings of the results. 

\item {\bf Theory assumptions and proofs}
    \item[] Question: For each theoretical result, does the paper provide the full set of assumptions and a complete (and correct) proof?
    \item[] Answer: \answerYes{} % Replace by \answerYes{}, \answerNo{}, or \answerNA{}.
    \item[] Justification: We clearly state the assumptions we make for the theorems. In particular, we refer to the data and augmentation distribution assumptions and clearly define the objective functions with their properties. 

    \item {\bf Experimental result reproducibility}
    \item[] Question: Does the paper fully disclose all the information needed to reproduce the main experimental results of the paper to the extent that it affects the main claims and/or conclusions of the paper (regardless of whether the code and data are provided or not)?
    \item[] Answer: \answerYes{} % Replace by \answerYes{}, \answerNo{}, or \answerNA{}.
    \item[] Justification: We state the details of our experimental setup. For the synthetic experiments, we give the recipe for data generation and for the real data, we describe how the data is randomly subsampled.

\item {\bf Open access to data and code}
    \item[] Question: Does the paper provide open access to the data and code, with sufficient instructions to faithfully reproduce the main experimental results, as described in supplemental material?
    \item[] Answer: \answerNo{} % Replace by \answerYes{}, \answerNo{}, or \answerNA{}.
    \item[] Justification: We are planning to release our code after we prepare a clean and distributable version which is also de-anonymized.

\item {\bf Experimental setting/details}
    \item[] Question: Does the paper specify all the training and test details (e.g., data splits, hyperparameters, how they were chosen, type of optimizer, etc.) necessary to understand the results?
    \item[] Answer: \answerYes{} % Replace by \answerYes{}, \answerNo{}, or \answerNA{}.
    \item[] Justification: We describe the experimental setup in details in the manuscript. We provide details on data generation, split and also properties of the algorithm(s) used.

\item {\bf Experiment statistical significance}
    \item[] Question: Does the paper report error bars suitably and correctly defined or other appropriate information about the statistical significance of the experiments?
    \item[] Answer: \answerYes{} % Replace by \answerYes{}, \answerNo{}, or \answerNA{}.
    \item[] Justification: Please see the Appendix for a complete overview of our results with statistical details.

\item {\bf Experiments compute resources}
    \item[] Question: For each experiment, does the paper provide sufficient information on the computer resources (type of compute workers, memory, time of execution) needed to reproduce the experiments?
    \item[] Answer: \answerYes{} % Replace by \answerYes{}, \answerNo{}, or \answerNA{}.
    \item[] Justification: We discuss the details related to compute resources in the Appendix.
    
\item {\bf Code of ethics}
    \item[] Question: Does the research conducted in the paper conform, in every respect, with the NeurIPS Code of Ethics \url{https://neurips.cc/public/EthicsGuidelines}?
    \item[] Answer: \answerYes{} % Replace by \answerYes{}, \answerNo{}, or \answerNA{}.
    \item[] Justification: We abide by the NeurIPS Code of Ethics to the best of our abilities.

\item {\bf Broader impacts}
    \item[] Question: Does the paper discuss both potential positive societal impacts and negative societal impacts of the work performed?
    \item[] Answer: \answerNo{} % Replace by \answerYes{}, \answerNo{}, or \answerNA{}.
    \item[] Justification: We study a fundamental theory problem in machine learning. Our work and results do not have any negative societal impact.
    
\item {\bf Safeguards}
    \item[] Question: Does the paper describe safeguards that have been put in place for responsible release of data or models that have a high risk for misuse (e.g., pretrained language models, image generators, or scraped datasets)?
    \item[] Answer: \answerNA{} % Replace by \answerYes{}, \answerNo{}, or \answerNA{}.
    \item[] Justification: Our work does not pose any such risks related to its content; we do not release new models or datasets.

\item {\bf Licenses for existing assets}
    \item[] Question: Are the creators or original owners of assets (e.g., code, data, models), used in the paper, properly credited and are the license and terms of use explicitly mentioned and properly respected?
    \item[] Answer: \answerYes{} % Replace by \answerYes{}, \answerNo{}, or \answerNA{}.
    \item[] Justification: Our theory and numerical experiments do not use any existing assets that requires license. Any academic resource, model, dataset and algorithm used are properly cited.

\item {\bf New assets}
    \item[] Question: Are new assets introduced in the paper well documented and is the documentation provided alongside the assets?
    \item[] Answer: \answerNA{} % Replace by \answerYes{}, \answerNo{}, or \answerNA{}.
    \item[] Justification: We do not release any assets that requires licensing and/or documentation.

\item {\bf Crowdsourcing and research with human subjects}
    \item[] Question: For crowdsourcing experiments and research with human subjects, does the paper include the full text of instructions given to participants and screenshots, if applicable, as well as details about compensation (if any)? 
    \item[] Answer: \answerNA{} % Replace by \answerYes{}, \answerNo{}, or \answerNA{}.
    \item[] Justification: Our paper does not involve crowdsourcing nor human subjects.

\item {\bf Institutional review board (IRB) approvals or equivalent for research with human subjects}
    \item[] Question: Does the paper describe potential risks incurred by study participants, whether such risks were disclosed to the subjects, and whether Institutional Review Board (IRB) approvals (or an equivalent approval/review based on the requirements of your country or institution) were obtained?
    \item[] Answer: \answerNA{} % Replace by \answerYes{}, \answerNo{}, or \answerNA{}.
    \item[] Justification: The theory-oriented scope of our work does not require such approvals.

\item {\bf Declaration of LLM usage}
    \item[] Question: Does the paper describe the usage of LLMs if it is an important, original, or non-standard component of the core methods in this research? Note that if the LLM is used only for writing, editing, or formatting purposes and does not impact the core methodology, scientific rigorousness, or originality of the research, declaration is not required.
    %this research? 
    \item[] Answer: \answerNA{} % Replace by \answerYes{}, \answerNo{}, or \answerNA{}.
    \item[] Justification: The development process does not involve LLM use as any important part of it.

\end{enumerate}

%%%%%%%%%%%%%%%%%%%%%%%%%%%%%%%%%%%%%%%%%%%%%%%%%%%%%%%%%%%%

\end{document}